\documentclass{article} % For LaTeX2e
\usepackage{iclr2026_conference,times}
\newcommand{\smallerthanSmall}{\fontsize{9pt}{9pt}\selectfont}
\usepackage[colorlinks=true, linkcolor=blue, urlcolor=blue, citecolor=blue]{hyperref}
\usepackage{hyperref}
\usepackage[utf8]{inputenc} % allow utf-8 input
\usepackage[T1]{fontenc}    % use 8-bit T1 fonts
\usepackage{hyperref}       % hyperlinks
\usepackage{url}    
\usepackage[most]{tcolorbox} % loads skins, listings, etc.
\usepackage{amsmath}

\newtcbox{\inlinebox}{on line,
  colback=gray!5, colframe=gray!80, boxrule=0.5pt, arc=2mm,
  left=1pt, right=1pt, top=0pt, bottom=0pt, boxsep=1pt,
  fontupper=\scriptsize}
\usepackage{alltt} % allows verbatim-like text with line wrapping
\usepackage{listings}
\usepackage{xcolor}% simple URL typesetting
\usepackage{booktabs}       % professional-quality tables
\usepackage{amsfonts}       % blackboard math symbols
\usepackage{nicefrac}       % compact symbols for 1/2, etc.
\usepackage{microtype}      % microtypography
\usepackage{xcolor}         % colors
\usepackage{lipsum}
\usepackage{graphicx} % Required for inserting images
\usepackage{tabularx}
\usepackage{amsmath, amssymb, amsthm}
\usepackage{enumitem}

\usepackage{graphicx}
\usepackage{cleveref}
\usepackage{makecell}
\usepackage{titlesec}
\usepackage{multirow}
\usepackage{wrapfig}
\usepackage[framemethod=TikZ]{mdframed}
\usepackage[utf8]{inputenc}
\usepackage{tikz}
\usepackage{setspace}
\usepackage{caption}
\usepackage{subcaption}
\usepackage{natbib}
\usepackage{tcolorbox}
\usepackage{float}
\usepackage{newfloat}
\usepackage{listings}
\usepackage{amsmath}
\usepackage{xcolor}
\usepackage{adjustbox}

\graphicspath{ {images/} }
\floatstyle{ruled}
\newfloat{listing}{tb}{lst}{}
\floatname{listing}{Listing}

\ExplSyntaxOn
\NewDocumentCommand{\mytilde}{}{
  \raisebox{-0.15ex}{$\mathtt{\sim}$}
}
\ExplSyntaxOff

\usepackage{tcolorbox}

\usepackage{algorithm}
\usepackage{algorithmic}
\usetikzlibrary{shapes.geometric, arrows, positioning, calc}
\tikzstyle{block} = [rectangle, draw, rounded corners, 
    minimum width=3cm, minimum height=1cm, text centered, font=\small]
\tikzstyle{decision} = [diamond, draw, aspect=2, inner sep=0.5em, text centered, font=\small]
\tikzstyle{arrow} = [thick,->,>=stealth]

\theoremstyle{plain}
\newtheorem{theorem}{Theorem}

\newtheorem{lemma}{Lemma}
\newtheorem{corollary}[theorem]{Corollary}
\theoremstyle{definition}
\newtheorem{definition}{Definition}

\theoremstyle{remark}

\title{Overcoming Joint Intractability with Lossless Hierarchical Speculative Decoding}

\author{
\parbox{\linewidth}{
\centering
\bfseries
Yuxuan Zhou$^{1}$\thanks{Work done during internship at Qwen Team, Alibaba Inc.},
Fei Huang$^{2}$,
Heng Li$^{1}$,
Fengyi Wu$^{3}$,
Tianyu Wang$^{3}$,\\
Jianwei Zhang$^{2}$,
Junyang Lin$^{2}$\thanks{Corresponding author.},
Zhi-Qi Cheng$^{3}$\footnotemark[\value{footnote}]\\[1em]
\normalfont
$^{1}$Independent Researcher \quad
$^{2}$Qwen Team, Alibaba Inc. \quad
$^{3}$University of Washington\\[0.5em]
\hspace*{-2em}\texttt{zhouyuxuanyx@gmail.com}, \texttt{junyang.ljy@alibaba-inc.com}, \texttt{zhiqics@uw.edu}
}
}

\iclrfinalcopy % Uncomment for camera-ready version, but NOT for submission.
\begin{document}

\maketitle
% \vspace{1em}

\begin{abstract}
Verification is a key bottleneck in improving inference speed while maintaining distribution fidelity in Speculative Decoding. Recent work has shown that sequence-level verification leads to a higher number of accepted tokens compared to token-wise verification. However, existing solutions often rely on surrogate approximations or are constrained by partial information, struggling with joint intractability. In this work, we propose \emph{Hierarchical Speculative Decoding (HSD)}, a provably lossless verification method that significantly boosts the expected number of accepted tokens and overcomes joint intractability by balancing excess and deficient probability mass across accessible branches. Our extensive large-scale experiments demonstrate that HSD yields consistent improvements in acceptance rates across diverse model families and benchmarks. Moreover, its strong explainability and generality make it readily integrable into a wide range of speculative decoding frameworks. Notably, integrating HSD into EAGLE-3 yields over a 12\% performance gain, establishing state-of-the-art decoding efficiency without compromising distribution fidelity. Code is available at \href{https://github.com/ZhouYuxuanYX/Hierarchical-Speculative-Decoding}{https://github.com/ZhouYuxuanYX/Hierarchical-Speculative-Decoding}.
\end{abstract}

\section{Introduction}
\label{sec:intro}

Inference speed has become paramount for Large Language Models (LLMs) \citep{achiam2023gpt, touvron2023llama, bai2023qwen}, which generate text auto-regressively. Recent advances in test-time scaling \citep{openai2024o1, guo2025deepseek, yu2025dapo, peng2025lmm} have further underscored its importance. While techniques like pruning \citep{frankle2018lottery, sun2023simple} and quantization \citep{shen2020q, xiao2023smoothquant} improve efficiency but sacrifice performance, Speculative Decoding \citep{leviathan2023fast} achieves speedups while preserving the target model’s distribution, making it a particularly appealing alternative. It adopts a smaller model to make proposals and a larger model to select from them with a grounded verification strategy. Most approaches prioritize the drafting phase, but further gains face diminishing returns. Driven by the verification bottleneck, recent methods \citep{cai2024medusa, zhou2024distillspec, narasimhan2024faster} trade off fidelity for speed, relying on task-specific tuning; their performance typically remains constrained to carefully curated scenarios.

Recent work \citep{sun2024block, qin2025optimized} shows that jointly verifying draft tokens can improve the expected number of accepted tokens, but faces joint intractability: simply applying the resampling strategy used in tokenwise verification \citep{leviathan2023fast} would require full joint probabilities over all possible decoding paths to correctly recover the target distribution, which is computationally infeasible. To address this, \citep{qin2025optimized} employs a lossy fixed acceptance threshold, while \citep{sun2024block} proposes Blockwise Verification, which provably recovers the target distribution. However, Blockwise Verification still falls short of the ideal case, and both its underlying mechanism and compatibility with other methods remain unclear.

% \begin{figure}[ht]
%     \centering
%     \includegraphics[width=0.8\linewidth]{sections/refined_ccdf.png}
%     \caption{Empirical complementary CDF of expected number of accepted tokens for diffrent verification algorithms with the draft length of 10. The draft and target distributions are the context-independent toy models introduced in \cite{sun2024block}.}
%     \label{fig:ccdf}
% \end{figure}

In this work, we propose Hierarchical Speculative Decoding (HSD), a provably lossless verification method built upon a novel hierarchical branch resampling strategy. In speculative decoding, resampling recovers portions of the target distribution that exceed the draft probability. As illustrated in \Cref{fig:overview}, HSD organizes multiple resampling distributions hierarchically across successive levels, with each distribution recovering only the partial target within its branch and resampling occurring immediately after the last accepted token. This design ensures the full target distribution is recovered in expectation while maximizing the expected number of accepted tokens, pushing the limits of lossless verification and enabling more efficient decoding. Notably, Blockwise verification focuses on independent verification with unclear potential for integration, while our method is designed to easily combine with other approaches, such as the widely adopted multi-draft setups.

% HSD pushes the limits of lossless verification by increasing the expected number of accepted tokens. 
% To highlight HSD’s advantages, we adopt the toy example with context-independent binary distributions from \cite{sun2024block} for illustration. Let $p(\cdot)$ and $q(\cdot)$ denote the target and draft distributions, respectively, we define: $p(A)\!=\!\frac{1}{3}, \quad p(B)\!=\!\frac{2}{3}, \quad q(A)\!=\!\frac{2}{3}, \quad q(B)\!=\!\frac{1}{3}.$
% With a draft length of 10, we run both algorithms for 10,000 iterations and plot the empirical complementary CDF of the acceptance length in \Cref{fig:overview}. HSD achieves a higher expected number of accepted tokens, especially due to the higher acceptance rate of longer drafts. 
% This advantage is also theoretically proven in \Cref{sec:length}. 

\begin{figure}
\centering
 \includegraphics[width=0.8\linewidth]{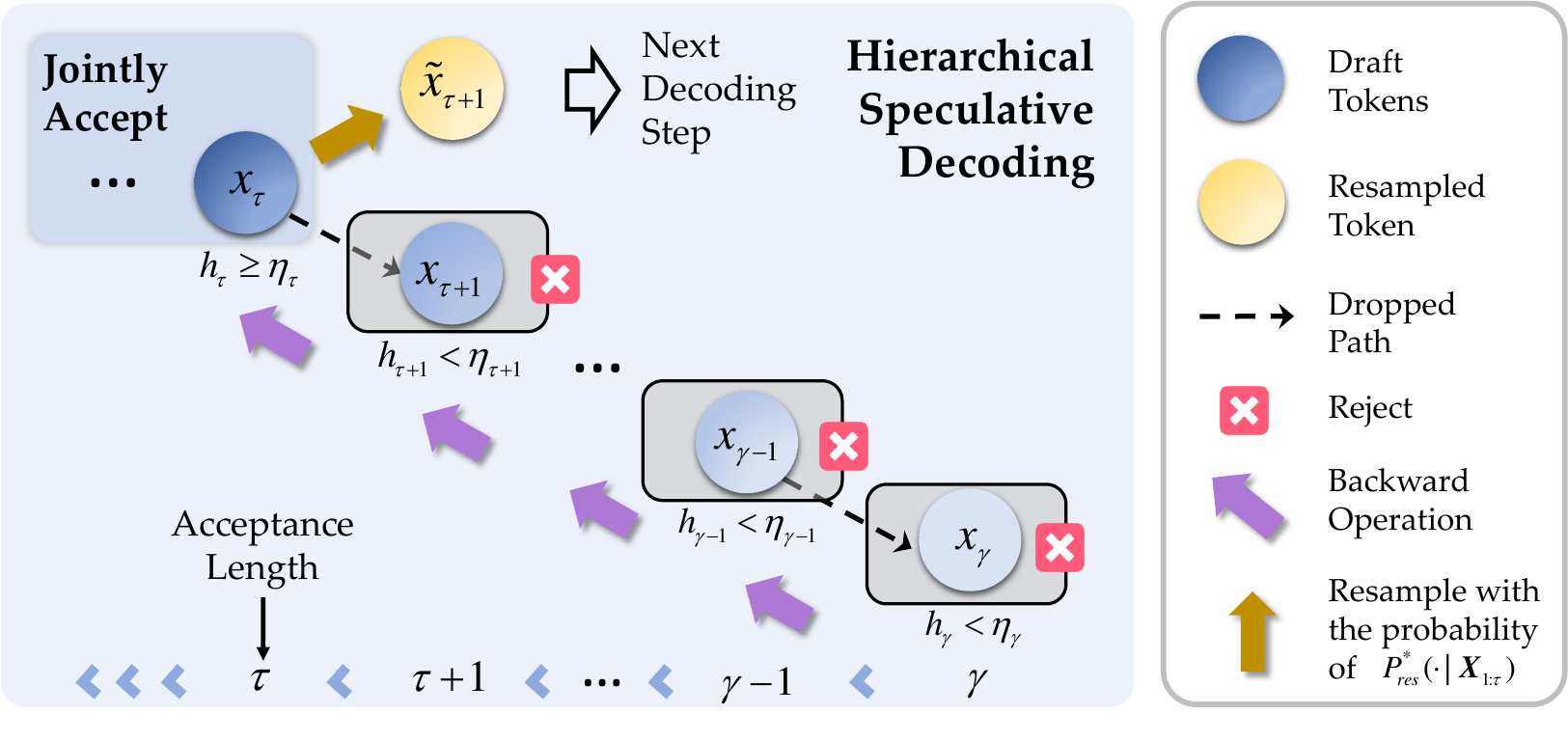}
\caption{\small{\textbf{Overview of HSD}. HSD accepts the draft \( \boldsymbol{X}_{\tau} \) by scanning backward from \( \gamma \) to \( \tau \), and then performs a single resampling at position \( \tau+1 \) using the corresponding distribution from the resampling hierarchy.}
}
% \vspace{-0.8cm}
\label{fig:overview}
\end{figure}

% \begin{wrapfigure}{r}{0.45\textwidth}
% \setlength{\abovecaptionskip}{-0.5cm} 
% \setlength{\belowcaptionskip}{-0.3cm} 
% \label{fig:ccdf}
% \centering
% \includegraphics[width=0.45\textwidth]{sections/refined_ccdf.png}
% \caption{Empirical complementary CDF of expected number of accepted tokens for different verification algorithms with the draft length of 10. The draft and target distributions are the context-independent toy models introduced in \cite{sun2024block}.}
% \vspace{-0.1cm}
% \end{wrapfigure}

\textbf{In summary, our contributions are as follows}:
\begin{itemize}
% \item We introduce \emph{Hierarchical Speculative Decoding} (HSD), a novel lossless verification method designed to be explainable and to integrate smoothly with existing speculative decoding frameworks, remaining largely orthogonal to them.

% \item We present a significant and practical advancement for inference scaling, yielding an average $6.7\%$ improvement in decoding speed across diverse benchmarks and model sizes. Importantly, these gains are achieved while preserving distributional fidelity, with efficiency improvements reaching up to $12.3\%$ on individual datasets.

% \item We empirically validate the effectiveness of HSD in large-scale settings, with particularly notable gains on longer sequences. Its average improvement $4.7\%$ in decoding speed within multidraft setups demonstrates its potential for integration with existing techniques.

\item We introduce \emph{Hierarchical Speculative Decoding} (HSD), a lossless and explainable verification method that integrates seamlessly with existing speculative decoding frameworks while remaining largely orthogonal to them.

\item HSD delivers a practical advance in inference scaling, achieving an average $6.7\%$ improvement in decoding speed across diverse benchmarks and model sizes while preserving distributional fidelity, with efficiency gains of up to $12.3\%$ on individual datasets.

\item HSD further improves decoding speed across multi-draft settings. Notably, integrating HSD into EAGLE-3 yields over $12\%$ performance gain, establishing new state-of-the-art decoding efficiency without compromising distribution fidelity.
\end{itemize}

% - why verify based on joint distribution inrease expected number of accepted tokens

% \begin{figure}
%     \centering
%     \includegraphics[width=0.8\linewidth]{alg_v2.pdf}
%     \caption{Pipeline of our method.}
%     \label{fig:alg_v1}
% \end{figure}

% \begin{wrapfigure}{r}{0cm}
% \label{fig:pipelien}
% \centering
% \includegraphics[width=0.5\textwidth]{alg_v2.pdf}
% \caption{\small \textbf{Backward-speculative verification pipeline.}
% The workflow instantiates the \emph{backward verifier} introduced in
% Section~\ref{sec:method-backward}.
% At each position \(k\) the token is tested by the acceptance rule
% \(h_k=\min\!\{\tfrac{p}{q},1\}\) in Eq.~\eqref{eq:accept_ratio}.  
% If \(h_k<\eta_k\) the token is \emph{rejected} (red outline) and redrawn from the
% resampling distribution \(p_{\mathrm{res}}(\tilde{x}_k\mid \boldsymbol{X}_{1:k-1},c)\) of
% Eq.~\eqref{eq:pres} (yellow box).  
% The scan proceeds right-to-left until the first index with
% \(r(x_k)\ge 1\);\;by Corollary~\ref{cor:branch} the entire prefix
% \(x_{1:k}\) can then be \emph{jointly accepted} (blue region), and earlier
% positions are skipped.  
% Dashed arrows mark discarded draft branches, solid arrows the final accepted
% trajectory.  
% The algorithm realises the inter-branch mass-redistribution principle proven in
% Theorem~\ref{thm:global-conservation}, thus preserving the exact target
% distribution \(p\) while amortising verification cost across the block.}
% \end{wrapfigure}

\section{Related Work}
\label{sec:related}
Follow-up research on speculative decoding ~\citep{leviathan2023fast}  can be organized into two main phases: the drafting phase and the verification phase.

\textbf{Drafting Phase.} Drafting methods can be grouped into three categories:
\emph{(1)~Single-draft.} Early SD methods~\citep{leviathan2023fast} inspired PaSS~\citep{monea2023pass} and Draft\&Verify~\citep{zhang2024draft}, improving efficiency via multi-token generation or selective layer skipping. GLIDE~\citep{du2024glide} (shared KV-cache) offers further speedups but requires task-specific tuning.
\emph{(2)~Retrieval-based.} LLM-A~\citep{yang2023LLMA} and ReST~\citep{he2023rest} generate drafts from reference texts, potentially reducing latency, but face database limitations, distribution gaps, and reliance on greedy decoding.
\emph{(3)~Multi-draft.} Tree-attention frameworks—SpecInfer~\citep{miao2024specinfer}, Medusa~\citep{cai2024medusa}, and Eagle~\citep{li2024eagle, fan2026flatter}—expand many branches, quickly exhausting memory. Medusa and Eagle also predict drafts from the target model’s hidden features rather than a separate draft model, further boosting speed but requires task-specific tuning.

\textbf{Verification Phase.} Verification methods trade fidelity for speed. Lossless approaches~\citep{sun2023spectr, yang2024multi, hu2025towards} guarantee exact recovery but are costly. Block Verification~\citep{sun2024block} partially alleviates this bottleneck but offers limited improvement and low interpretability and integrity. Lossy methods—including BiLD~\citep{kim2023bild}, MTAD~\citep{qin2025optimized}, DistillSpec~\citep{zhou2024distillspec}, Medusa-2~\citep{cai2024medusa}, SpecCascade~\citep{narasimhan2024faster} and CoS~\citep{fu2025fast} increase speed but compromise distribution fidelity and require task-specific tuning.
In addition, Medusa and EAGLE always accept the first draft token to improve throughput, trading off exact recovery of the target distribution.

\section{Revisiting Tokenwise Speculative Decoding}
\label{sec:tokenwise}

In tokenwise speculative sampling~\citep{leviathan2023fast}, each token $x_t$ is drafted from $q(x_t)$ and verified against $p(x_t)$. It is accepted with probability $h(x_t)=\min\{1,\,p(x_t)/q(x_t)\}$, or rejected and replaced from $P_{\text{res}}(x_t)$. Thus the probability that $x_t$ is finally produced (``yielded'') is:
\begin{align}
P(x_t \text{ yielded})
&= P(x_t \text{ drafted and accepted})
 + P(\tilde{x}_t \text{ drafted and rejected},\, x_t \text{ resampled}). \label{eq:tokenwise-yield}
\end{align}

\paragraph{Accept term.}
If $x_t$ is proposed by $q$ and accepted,
\begin{equation}
P(x_t \text{ drafted and accepted})
  = q(x_t)\,h(x_t)
  = q(x_t)\,\min\!\{1,\,p(x_t)/q(x_t)\}.
\end{equation}
\paragraph{Resampling term.}
When a draft $\tilde{x}_t$ is rejected, the verifier resamples from
\[
P_{\text{res}}(x_t)
  = \frac{p(x_t) - \min\{p(x_t),q(x_t)\}}
         {\sum_{\tilde{x}_t\in \mathcal{V}} \bigl(p(\tilde{x}_t) - \min\{p(\tilde{x}_t),q(\tilde{x}_t)\}\bigr)}.
\]
The total probability of rejection is $\sum_{\tilde{x}_t \in \mathcal{V}} q(\tilde{x}_t)(1-h(\tilde{x}_t))$, giving
\begin{equation}
P(\tilde{x}_t \text{ drafted and rejected},\, x_t \text{ resampled})
   = \Bigl[\sum_{\tilde{x}_t\in \mathcal{V}} q(\tilde{x}_t)(1-h(\tilde{x}_t))\Bigr] P_{\text{res}}(x_t).
\end{equation}
\paragraph{Final distribution.}
The sum $\sum_{\tilde{x}_t \in \mathcal{V}} q(\tilde{x}_t)\,(1-h(\tilde{x}_t))$ corresponds to the total \emph{excess mass} assigned by the draft distribution to tokens where it allocates more probability than the target, while the denominator of $P_{\text{res}}(x_t)$ measures the total \emph{deficient mass}, i.e., the probability assigned by the target to tokens where it allocates more than the draft. For tokenwise distributions these match ($D_{\mathrm{LK}}(q,p) = D_{\mathrm{LK}}(p,q)$), so they cancel, yielding
\[
P(x_t \text{ is yielded}) = q(x_t) h(x_t) + D_{\mathrm{LK}}(q,p)\frac{p(x_t)-q(x_t)h(x_t)}{D_{\mathrm{LK}}(p,q)} = p(x_t).
\]

\section{Theoretical Foundations of Hierarchical Speculative Decoding}
\label{sec:pre}

For any \textbf{lossless speculative decoding}, the probability of generating an output decomposes into two parts: (1) the probability a draft is \emph{accepted}, becoming the final output, and (2) the probability a draft is \emph{rejected}, triggering a corrective resampling step. In \emph{token-wise} speculative decoding, resampling is straightforward because each token's probability is directly accessible. In contrast, full joint probabilities over sequences are intractable for auto-regressive models. \textbf{Hierarchical Speculative Decoding (HSD)} overcomes this via \emph{hierarchical branch resampling}, where multiple resampling distributions at different levels recover \emph{partial target distributions}, which together statistically recover the full distribution. This section formalizes the theoretical foundations.

\subsection{Recovery of Partial Distributions}
\label{sec:recovery}
To guide recovery within accessible subsets, we extend the divergence from \cite{leviathan2023fast} to partial distributions. Let \(\omega\) be a token or sequence, \(\Omega\) the full sample space, and \(p(\cdot), q(\cdot)\) the target and draft distributions. For \(\Omega'\subseteq \Omega\), define the \emph{generalized divergence}:

\begin{definition}
\label{def:general}
\textbf{Generalized Divergence.}
Given two distributions \( p \) and \( q \) over a sample space \( \Omega \), and a subset \( \Omega' \subseteq \Omega \), the \emph{generalized divergence} over \( \Omega' \) is defined as:
\begin{equation}
\setlength{\abovedisplayskip}{2pt}
\setlength{\belowdisplayskip}{2pt}
D_{\Omega'}(p, q) = \sum_{\tilde{\omega} \in \Omega'} \max\{p(\tilde{\omega}) - q(\tilde{\omega}),\ 0\}.
\end{equation}
The \emph{generalized divergence} \(D_{\Omega'}(p,q)\) measures the total \emph{deficient mass}, i.e., how much probability mass is missing in the draft \(q\) relative to the target \(p\) within the subset \(\Omega'\). The reverse divergence \(D_{\Omega'}(q,p)\) measures the corresponding \emph{excess mass}. In the whole space \(\Omega\), this is symmetric  (see \Cref{lem:symmetry} in \Cref{app:symmetry_TD} ) and reduces to the divergence from \cite{leviathan2023fast} (see \Cref{lem:equivalence} in \Cref{sec:equivalence}), which underpins standard token-wise speculative decoding.
\end{definition}

% This symmetry underpins the correctness of speculative decoding, as will be discussed later.

% Next, we define the \emph{generalized resampling distribution}:
% \begin{definition}\label{def:resample}
% \textbf{Generalized Resampling Distribution.}
% Consider a verification scheme that uses an acceptance probability
% $h(\omega) \;=\; \min\!\Bigl\{\tfrac{p(\omega)}{q(\omega)},\,1\Bigr\}$, resample only occurs if the draft token or sequence $\omega$ is rejected. Over a subset \(\Omega' \subseteq \Omega\), the \emph{generalized resampling distribution} is given by
% \begin{equation}
% \setlength{\abovedisplayskip}{2pt}
% \setlength{\belowdisplayskip}{2pt}
% p_{\mathrm{res}}(w\mid \Omega')
% \;=\;
% \frac{\max\{p(w)\,-\,q(w),\;0\}}
%      {\max\{\,D_{\Omega'}(p,q),\;D_{\Omega'}(q,p)\}},
% \end{equation}
% where \(D_{\Omega'}(q,p)\) represents the total \emph{trigger mass} for resampling, and \(D_{\Omega'}(p,q)\) is the mass that needs to be recovered. Particularly, when \(\Omega'=\Omega\) spans the entire sample space, this rule coincides with the classical resampling distribution introduced by Leviathan~\cite{leviathan2023fast} (see \Cref{lem:equivalence} in the appendix for more details). 
% \end{definition}

Next, we formalize the condition under which the partial target distribution is fully recoverable:

\begin{theorem}\label{theorem:derivation}
\textbf{Partial Distribution Recovery.} 
A target distribution over \(\Omega'\subseteq \Omega\) can be fully recovered via resampling iff $D_{\Omega'}(p, q)\,\le\;D_{\Omega'}(q, p)$. (See proof in \Cref{app:proof_partial_recovery}.)
\end{theorem}
Intuitively, this ensures the "trigger mass" in the draft is sufficient to compensate for the deficit in the target distribution. Over the full space \(\Omega\), symmetry guarantees full recoverability.

\subsection{Resampling within the Accessible Branch}
\label{sec:intra}

With these definitions, we analyze resampling within \emph{accessible branches} along a draft sequence. Although computing full joint probabilities is intractable, the probabilities of all next tokens over the vocabulary $\mathcal{V}$ are accessible given any prefix $\boldsymbol{X}_{1:t-1}$. We define a \emph{branch} as:

\begin{equation}
   \text{Branch}(\boldsymbol{X}_{1:t-1}) = \{ \boldsymbol{X}_{1:t} = (\boldsymbol{X}_{1:t-1}, \tilde{x}_t) \mid \tilde{x}_t \in \mathcal{V} \}. 
\end{equation}
Branch divergence will guide redistribution of excess probability mass to correct local deficits.

Since only joint probabilities $p(\boldsymbol{X}_{1:t})$ within a given branch $\text{Branch}(\boldsymbol{X}_{1:t-1})$ are available, we introduce \emph{branch divergence} to quantify local deficits in the draft:

\begin{definition}
\label{def:branch_divergence}
    \textbf{Branch Divergence} 
\begin{equation}
    \setlength{\abovedisplayskip}{2pt}
\setlength{\belowdisplayskip}{2pt}
         D_\text{Branch}(p, q \mid \boldsymbol{X}_{1:t-1}) =  \sum_{\boldsymbol{X}_{1:t}\in \text{Branch}(\boldsymbol{X}_{1:t-1})} \text{max}\{p\left(\boldsymbol{X}_{1:t}\right) - q\left(\boldsymbol{X}_{1:t}\right), 0\}  
\end{equation}
\end{definition}

Branch divergence captures how much probability mass is missing locally. Unlike total divergence, it is inherently asymmetric, motivating the definition of \emph{branch asymmetry}:

\begin{definition} 
\label{def:asym}
\textbf{Asymmetry of Branch Divergence} 
\begin{equation}
\setlength{\abovedisplayskip}{2pt}
\setlength{\belowdisplayskip}{2pt}
    \Delta_\text{Branch}(\boldsymbol{X}_{1:t-1})  = D_\text{Branch}(p, q \mid \boldsymbol{X}_{1:t-1}) - D_\text{Branch}(q, p \mid \boldsymbol{X}_{1:t-1})
\end{equation}
\end{definition}
Asymmetry essentially reflects the probabilistic imbalance within the current branch. Here, $\Delta_\text{Branch} > 0$ indicates a deficit that cannot be corrected within the branch alone, while $\Delta_\text{Branch} < 0$ represents excess mass available to support other branches. It can be computed as follows:

\begin{theorem} 
\label{theorem:branch}
Quantifying Asymmetry of Branch Divergence     (see proof in \Cref{proof:quantify}):
\begin{equation}
   \setlength{\abovedisplayskip}{2pt}
\setlength{\belowdisplayskip}{2pt}
    \Delta_\text{Branch}(\boldsymbol{X}_{1:t-1})  = p\left(\boldsymbol{X}_{1:t-1}\right) - q\left(\boldsymbol{X}_{1:t-1}\right),        
\end{equation}
\end{theorem}

From \Cref{theorem:derivation} and \Cref{theorem:branch}, we conclude that resampling can fully recover the target distribution over a branch whenever the draft has enough probability mass to cover the deficit:

\begin{corollary}
\label{cor:resample}
The target distribution over the Branch($\boldsymbol{X}_{1:t-1}$) can be recovered via resampling, under the following condition:
\begin{equation}
\setlength{\abovedisplayskip}{2pt}
\setlength{\belowdisplayskip}{2pt}
p(\boldsymbol{X}_{1:t-1}) \leq q(\boldsymbol{X}_{1:t-1}) \quad \text{or, equivalently,} \quad r(\boldsymbol{X}_{1:t-1}) \leq 1
\end{equation}
where $r\left(\boldsymbol{X}_{1:t-1}\right) = \frac{p\left(\boldsymbol{X}_{1:t-1}\right)}{q\left(\boldsymbol{X}_{1:t-1}\right)}$ denotes the probability ratio.
\end{corollary}

For drafts of length $\gamma$, the full target distribution cannot be recovered by applying verification solely within the accessible $\text{Branch}(\boldsymbol{X}_{1:\gamma-1})$. However, we observe that the unused probability mass in certain branches can be leveraged to compensate for the unrecoverable mass in other branches, from a statistical perspective. This motivates the hierarchical branch resampling approach discussed next.

\subsection{Resampling in a Hierarchy of Accessible Branches}
\label{sec:inter}

Accessible branch divergences naturally form a hierarchical structure that enables systematic redistribution of excess probability mass. Specifically:
 
\begin{theorem} 
\label{thm:hierarchy}
\textbf{Hierarchy of Branch Divergence}

The total positive asymmetry of branch divergence across child branches is equal to the parent branch divergence, and vice versa. Specifically:

\begin{equation}
\setlength{\abovedisplayskip}{2pt}
\setlength{\belowdisplayskip}{2pt}
    \sum_{\Delta_\text{Branch}(\boldsymbol{X}_{1:t-2}, \tilde{x}_{t-1})>0}  \Delta_\text{Branch}(\boldsymbol{X}_{1:t-2}, \tilde{x}_{t-1}) =    D_\text{Branch}(p, q \mid \boldsymbol{X}_{1:t-2}), \quad \text{and vice versa,}
\end{equation}
where $\boldsymbol{X}_{1:t-2}, \tilde{x}_{t-1}$ ranges over all possible Branches with the shared prefix $\boldsymbol{X}_{1:t-2}$,  and $\boldsymbol{X}_{1:t-2}$ is the accessible branch along the draft sequence.
(See \Cref{proof:hierarchy} for the proof.)
\end{theorem}

This result guarantees that excess mass from overrepresented branches can be aggregated to offset deficits in underrepresented branches. Thus, hierarchical branch resampling guarantees exact recovery of the target distribution, even when individual branches cannot. This provides a rigorous theoretical foundation for deriving Hierarchical Speculative Decoding.

\begin{figure}[ht]
\centering
\begin{minipage}[t]{0.49\linewidth}
\begin{algorithm}[H]
\setstretch{1}
\caption{Naive HSD}
\label{alg:naive}
\smallerthanSmall
\begin{algorithmic}[1]
\REQUIRE Target probabilities: $\{p(\cdot),..., p(\cdot|\boldsymbol{X}_{1:\gamma})\}$
\REQUIRE Draft probabilities: $\{q(\cdot),..., q(\cdot|\boldsymbol{X}_{1:\gamma-1})\}$
\REQUIRE Draft tokens $\boldsymbol{X}_{1:\gamma}=\{x_1, ..., x_\gamma \}$

\STATE Initialize $\tau = 0$
\FOR{$t$ \textbf{in} $\gamma:1$} 
    \STATE Sample $\eta_t \sim U(0,1)$
    \IF{$h_t \geq \eta_t$}
        \STATE Set $\tau = t$  \hfill {\color{blue}\textit{\#accept $\boldsymbol{X}_{1:t}$}}
        \STATE \textbf{break}
    \ELSE
        \STATE Set $\tau = t-1$  \hfill {\color{blue}\textit{\#reject $x_{t}$}}
        \STATE \textbf{continue}  \hfill {\color{blue}\textit{\#step back}}
    \ENDIF
\ENDFOR

\IF{$\tau = \gamma$}
    \STATE Sample token from $p(\cdot|\boldsymbol{X}_{1:\gamma})$  \hfill {\color{blue}\textit{\#bonus token}}
\ELSE 
    \FOR{$t$ \textbf{in} $\tau:\gamma-1$}
        \STATE Sample token from $P_{\text{res}}(\cdot \mid \boldsymbol{X}_{1:t})$ \hfill {\color{blue}\textit{\#resample}}
    \ENDFOR
\ENDIF

\ENSURE $[\boldsymbol{X}_{1:\tau}, \Tilde{x}_{\tau+1}, \dots, \Tilde{x}_\gamma]$
\end{algorithmic}
\end{algorithm}
\end{minipage}
\hfill
\begin{minipage}[t]{0.49\linewidth}
\begin{algorithm}[H]
\setstretch{1}
\caption{HSD}
\label{alg:backward}
\smallerthanSmall
\begin{algorithmic}[1]
\REQUIRE Target probabilities: $\{p(\cdot),..., p(\cdot|\boldsymbol{X}_{1:\gamma})\}$
\REQUIRE Draft probabilities: $\{q(\cdot),..., q(\cdot|\boldsymbol{X}_{1:\gamma-1})\}$
\REQUIRE Draft tokens $\boldsymbol{X}_{1:\gamma}=\{x_1, ..., x_\gamma \}$

\STATE Initialize $\tau = 0$

\FOR{$t$ \textbf{in} $\gamma:1$}
    % \STATE Compute acceptance probability $h_t$ from \Cref{def:acceptance_probability}
    \STATE Sample $\eta_t \sim U(0,1)$
    \IF{$h_t \geq \eta_t$}
        \STATE Set $\tau = t$ \hfill {\color{blue}\textit{\#accept $\boldsymbol{X}_{1:t}$}}
        \STATE \textbf{break}
    \ELSE
        \STATE Set $\tau = t-1$ \hfill {\color{blue}\textit{\#reject $x_{t}$}}
        \STATE \textbf{continue} \hfill {\color{blue}\textit{\#step back}}
    \ENDIF
\ENDFOR

\IF{$\tau = \gamma$}
    \STATE Sample token from $p(\cdot|\boldsymbol{X}_{1:\gamma})$ \hfill {\color{blue}\textit{\#bonus token}}
\ELSE
    \STATE Sample token from $P^*_{\text{res}}(\cdot \mid \boldsymbol{X}_{1:\tau})$ \hfill {\color{blue}\textit{\#resample}}
\ENDIF

\ENSURE $[\boldsymbol{X}_{1:\tau}, \text{token}]$
\end{algorithmic}
\end{algorithm}
\end{minipage}
\end{figure}

\section{Hierarchical Speculative Decoding}
Guided by the theoretical foundations, we first develop a \emph{naive algorithm} (see \ref{app: naive algorithm}) that exactly recovers the target distribution. The procedure evaluates a candidate sequence \( \boldsymbol{X}_{1:\gamma} \) and scans backward to identify the longest accepted prefix \( \boldsymbol{X}_{1:\tau} \), then recursively resamples positions \( \tau + 1 \) through \( \gamma \) using the corresponding distributions from the resampling hierarchy. 

This naive approach, however, still requires \( \gamma - \tau + 1 \) additional calls to the target model, since the resampled branches are inaccessible. To remove this overhead, we introduce \emph{Capped Branch Resampling}, yielding our final \emph{Hierarchical Speculative Decoding (HSD)}. HSD recovers the target distribution with just one resampling step within the accessible branches. Concretely, after the resampling step at line 15 in \Cref{alg:backward}, HSD only needs to sample from the target distribution to continue generation until $\gamma$, which can be replaced by another speculative decoding step, eliminating additional target calls.

\subsection{Naive Hierachicial Speculative Decoding}
\label{app: naive algorithm}

Specifically, the acceptance probability is computed according to the following formula:
{
\begin{tcolorbox}[
    colback=blue!5!white,  % 背景颜色：5% 黑色，95% 白色 (浅灰色)
    colframe=blue!45!white, % 边框颜色：75% 黑色 (深灰色)
    arc=1mm,                 % 圆角半径
    boxrule=1.5pt,           % 边框线粗细
    breakable,               % 允许方框跨页
    enhanced,                % 允许更高级的绘制命令
    % 可选：为方框内的数学环境优化间距
    before=\vspace{2pt},
    after=\vspace{-5pt},
]
\setstretch{1}
% here we define r^*, two h, and p_res. 
\textbf{\emph{Acceptance Probability}} $h_\gamma = \min\{r(\boldsymbol{X}_{1:\gamma}), 1\}$, and when $t<\gamma$:
\begin{equation}
\label{eq:acceptance_probability_naive}
h_t =
% \begin{cases}
    \frac{D_{\text{Branch}}(p, q \mid \boldsymbol{X}_{1: t})}{\max \{ D_{\text{Branch}}(p, q \mid \boldsymbol{X}_{1: t}), D_{\text{Branch}}(q, p \mid \boldsymbol{X}_{1: t})\}},  
% \end{cases}
\end{equation}
\\
\textbf{\emph{Branch Resampling Probability}} (line 17 in \Cref{alg:naive}):
\begin{equation}
\label{eq:resample_prob_naive}
\setlength{\abovedisplayskip}{0pt}
\setlength{\belowdisplayskip}{0pt}
P_{r e s}\left(x_t \mid \boldsymbol{X}_{1:t-1}\right)=\frac{\max \left\{p\left(\boldsymbol{X}_{1: t}\right) - q\left(\boldsymbol{X}_{1: t}\right) , 0\right\}}{D_{\text{Branch}}(p, q \mid \boldsymbol{X}_{1: t-1})}
\end{equation}
\end{tcolorbox}
}
\emph{Branch Divergence} $D_{\text{Branch}}(p, q \mid \boldsymbol{X}_{1: t-1})$ is defined in \Cref{def:branch_divergence}. By construction, the \emph{Branch Resampling Probability} is defined within the accessible Branch($\boldsymbol{X}_{1:t-1}$), i.e., $P_\text{res}(\boldsymbol{X}_{1:t}\mid\text{Branch}(\boldsymbol{X}_{1:t-1}))$, which reduces to the token-level form $P_\text{res}(x_{t}\mid\boldsymbol{X}_{1:t-1})$.

The probability of the Target Model generating a sequence $\boldsymbol{X}_{1:\gamma}$ can be decomposed into two disjoint events: (i) full acceptance of the draft, or (ii) at least one rejection followed by resampling:
\begin{equation}
\begin{aligned}
P\bigl(\boldsymbol{X}_{1:\gamma}\,\text{is yielded}\bigr) \;  =& \; P\bigl(\boldsymbol{X}_{1:\gamma}\,\text{is sampled as draft},\,\boldsymbol{X}_{1:\gamma} \text{ is accepted}\bigr) \\
&+\sum_{\tilde{\boldsymbol{X}}_{1:\gamma} \neq \boldsymbol{X}_{1:\gamma}} 
P(\tilde{\boldsymbol{X}}_{1:\gamma} \text{ sampled and rejected, } \boldsymbol{X}_{1:\gamma} \text{ resampled}).
\end{aligned}
\end{equation}
\noindent\underline{Accept term}: probability for the case when $\boldsymbol{X}_{1:\gamma}$ is sampled as draft and then directly accepted.
\begin{equation}
\begin{aligned}
    & P\bigl(\boldsymbol{X}_{1:\gamma}\,\text{is sampled as draft},\,\boldsymbol{X}_{1:\gamma} \text{ is accepted}\bigr) 
    % &\quad = P\bigl(\boldsymbol{X}_{1:\gamma}\,\text{is sampled as draft}\bigr)  \underbrace{P\bigl(\boldsymbol{X}_{1:\gamma} \text{ is accepted}\mid \boldsymbol{X}_{1:\gamma}\,\text{is sampled as draft}\bigr)}_{\text{The probability of accepting}\boldsymbol{X}_{1:\gamma} \; \text{at} \; \gamma \; \text{position}}.\\
  =     % q(\boldsymbol{X}_{1:\gamma}) h_\gamma 
    \underbrace{q(\boldsymbol{X}_{1:\gamma})}_{\text{sample probability}} \underbrace{\min \{r(\boldsymbol{X}_{1:\gamma}), 1\}}_{\text{accept probability at} \; \gamma}
\end{aligned}
\end{equation}
If $r(\boldsymbol{X}_{1:\gamma}) \leq 1$, this equals to the target probability $p(\boldsymbol{X}_{1:\gamma})$. Otherwise, it is equal to $q(\boldsymbol{X}_{1:\gamma})$, and the residual probability $p(\boldsymbol{X}_{1:\gamma}) - q(\boldsymbol{X}_{1:\gamma})$ is compensated via resampling.
% \medskip

\noindent\underline{Resampling term (partially resampled)}:
This term accounts for all cases where $\boldsymbol{X}_{1:\gamma}$ is obtained by resampling. Note that the accepted prefix must exactly match the corresponding subsequence of $\boldsymbol{X}_{1:\gamma}$ for this contribution to apply. Therefore, we can further decompose it by summing over all possible positions $\tau+1$ of the first rejected token, with $\tau$ being the length of the longest accepted prefix:
% \begin{equation}
% \small{
% \begin{aligned}
% \!\!\!&\!\!\!\sum_{\tau\!=\!0}^{\gamma} \sum_{\tilde{\boldsymbol{X}}_{\tau+1:\gamma}} 
% \!\!\!\!\!P(
% \boldsymbol{X}_{1:\tau}\tilde{\boldsymbol{X}}_{\tau+1:\gamma} \text{ is sampled as draft},
% \tilde{\boldsymbol{X}}_{\tau\!+\!1:\gamma} \text{ is rejected}, 
% \boldsymbol{X}_{1:\tau} \text{is accepted},
% \boldsymbol{X}_{\tau+1:\gamma}\text{ is resampled}) \\
% &\!=\!\sum_{\tau=0}^{\gamma} \sum_{\boldsymbol{X}_{1:\tau}, \tilde{\boldsymbol{X}}_{\tau+1:\gamma}}
% \underbrace{q(\boldsymbol{X}_{1:\tau}\tilde{\boldsymbol{X}}_{\tau+1:\gamma})}_{\text{draft probability}} \cdot\underbrace{\prod_{t=\tau+1}^\gamma (1-h_t)}_{\text{reject backwards from } \gamma \text{ to } \tau+1} \quad \cdot \underbrace{h_\tau}_{\text{accept } \boldsymbol{X}_{1:\tau}} \cdot \underbrace{\prod_{t=\tau+1}^\gamma P_{\text{res}}(x_t)}_{\text{resample from } \tau+1 \text{ to } \gamma}
% \end{aligned}}
% \label{eq:resample}
% \end{equation}
% \tianyu{
\begin{equation}
\begin{aligned}
&\sum_{\tau\!=\!0}^{\gamma} \sum_{\tilde{\boldsymbol{X}}_{\tau+1:\gamma}} 
P(\tilde{\boldsymbol{X}}_{1:\gamma} \text{ sampled and rejected, } \boldsymbol{X}_{1:\gamma} \text{ resampled}) = \\
% &\sum_{\tau=0}^{\gamma} \sum_{\tilde{\boldsymbol{X}}_{\tau+1:\gamma}}
% \underbrace{q(\boldsymbol{X}_{1:\tau}\tilde{\boldsymbol{X}}_{\tau+1:\gamma})}_{\text{(1) draft probability}} \cdot\underbrace{\prod_{t=\tau+1}^\gamma (1-h_t)}_{\text{(2) reject backwards from } \gamma \text{ to } \tau+1} \quad \cdot \underbrace{h_\tau}_{\text{(3) accept } \boldsymbol{X}_{1:\tau}} \cdot \underbrace{\prod_{t=\tau+1}^\gamma P_{\text{res}}(x_t)}_{\text{(4) resample from } \tau+1 \text{ to } \gamma}
&\sum_{\tau=0}^{\gamma} \sum_{\tilde{\boldsymbol{X}}_{\tau+1:\gamma}}
q(\boldsymbol{X}_{1:\tau}\tilde{\boldsymbol{X}}_{\tau+1:\gamma}) \cdot\prod_{t=\tau+1}^\gamma (1-h_t)\quad \cdot h_\tau \boldsymbol{X}_{1:\tau} \cdot \prod_{t=\tau+1}^\gamma P_{\text{res}}(x_t)
\end{aligned}
\label{eq:resample}
\end{equation}
% }

\noindent
\textbf{Explanation of terms:}
\begin{enumerate}
    \item \textbf{Sampling:} $q(\boldsymbol{X}_{1:\tau} \tilde{\boldsymbol{X}}_{\tau+1:\gamma})$ is the probability of generating the initial draft sequence.
    \item \textbf{Backward Scan:} $\prod_{t=\tau+1}^{\gamma} (1 - h_t)$ corresponds to scanning backward from the end, rejecting tokens until the first accepted prefix is found.
    \item \textbf{Acceptance:} $h_\tau$ is the probability of accepting the longest prefix $\boldsymbol{X}_{1:\tau}$.
    \item \textbf{Resampling:} $\prod_{t=\tau+1}^{\gamma} P_{\text{res}}(x_t)$ resamples the remaining positions to recover exactly the target probability.
\end{enumerate}

This decomposition defines the procedure underlying \Cref{alg:naive} and provides the basis for its provable losslessness. The complete proof is given in \Cref{proof:naive-proof}, together with an illustrative example \Cref{proof:naive-example} showing how naive HSD recovers the target distribution.

\subsection{Hierarchical Speculative Decoding with Capped Branch Resampling}
\label{sec:method}

To introduce the capped branch sampling, we first define the \emph{Maximum Prefix Ratio Index}.

\begin{definition}
\label{def:max_ratios}
\textbf{Maximum Prefix Ratio Index}
For candidate tokens $\boldsymbol{X}_{1:t}$, the \emph{Maximum Prefix Ratio Index} $m(\boldsymbol{X}_{1:t})$ is the position in the prefix $\boldsymbol{X}_{1:t-1}$ where the joint probability ratio $r(\boldsymbol{X}_{1:i})$ is maximized; if no prefix exceeds 1, we set $m(\boldsymbol{X}_{1:t}) = 0$:
\[
m(\boldsymbol{X}_{1:t}) = 
\arg\max_{1 \le i < t} r(\boldsymbol{X}_{1:i}) \;\;\text{or}\;\; 0 \text{ if } \max_{1 \le i < t} r(\boldsymbol{X}_{1:i}) \le 1.
\]
\end{definition}

Based on the \emph{Maximum Prefix Ratio Index}, we define the \emph{Capped Prefix Ratio} $r^*$ as follows:
\begin{definition}
\label{def:r_star}
\textbf{Capped Prefix Ratio}
\begin{equation}
    \label{eq:r_star}
r^{*}(\boldsymbol{X}_{1:t}) = \min\{r(\boldsymbol{X}_{1:m(\boldsymbol{X}_{1:t})}), 1\}r(\boldsymbol{X}_{m(\boldsymbol{X}_{1:t})+1:t}).
\end{equation}
By Definition~5, we have $r(\boldsymbol{X}_{1:m(\boldsymbol{X}_{1:t})}) >1$, and according to \Cref{eq:r_star}, this implies the identity $r^*(\boldsymbol{X}_{1:t}) = r\bigl(\boldsymbol{X}_{m(\boldsymbol{X}_{1:t})+1:t}\bigr)$. 
\end{definition}

Then we define the \emph{Capped Branch Divergence}:
\begin{definition}
\label{def:d_star}
    \textbf{Capped Branch Divergence}
\begin{align}
\label{eq:d_star}
D^*_{\text{Branch}}\left(p, q \mid \boldsymbol{X}_{1: t-1}\right) &= \sum_{\substack{\boldsymbol{X}_{1:t} \in \text{Branch}(\boldsymbol{X}_{1:t-1}); \\ r^*\left(\boldsymbol{X}_{1: t}\right)>1}} \left(r^*\left(\boldsymbol{X}_{1: t}\right)-1\right) q\left(\boldsymbol{X}_{1: t}\right) & \\D^*_{\text{Branch}}\left(q, p \mid \boldsymbol{X}_{1: t-1}\right) &= \sum_{\substack{\boldsymbol{X}_{1:t} \in \text{Branch}(\boldsymbol{X}_{1:t-1}); \\ r^*\left(\boldsymbol{X}_{1: t}\right)\leq1}} \left(1-r^*\left(\boldsymbol{X}_{1: t}\right)\right) q\left(\boldsymbol{X}_{1: t}\right) \label{eq:side_by_side_align}
\end{align}
\end{definition}
Finally, the acceptance probability is computed according to the following formula:
{
\begin{tcolorbox}[
    colback=blue!5!white,  % 背景颜色：5% 黑色，95% 白色 (浅灰色)
    colframe=blue!55!white, % 边框颜色：75% 黑色 (深灰色)
    arc=1mm,                 % 圆角半径
    boxrule=1.5pt,           % 边框线粗细
    breakable,               % 允许方框跨页
    enhanced,                % 允许更高级的绘制命令
    % 可选：为方框内的数学环境优化间距
    before=\vspace{2pt},
    after=\vspace{2pt},
]
\setstretch{1}
\textbf{\emph{Acceptance Probability}} $h_\gamma = \min \{r^*(\boldsymbol{X}_{1:\gamma}), 1\}$, and when $t<\gamma$:
\begin{equation}
\label{def:acceptance_probability}
h_t =
% \begin{cases}
    \frac{D^*_{\text{Branch}}(p, q \mid \boldsymbol{X}_{1: t})}{D^*_{\text{Branch}}(q, p \mid \boldsymbol{X}_{1: t})},  
% \end{cases}
\end{equation}

\textbf{\emph{Capped Branch Resampling Probability}}~~(line 15 in \Cref{alg:backward}):
\begin{equation}
\label{eq:cap_resample}
\setlength{\abovedisplayskip}{0pt}
\setlength{\belowdisplayskip}{0pt}
P^*_\text{res}\left(x_t \mid \boldsymbol{X}_{1:t-1}\right)=\frac{\max \left\{q\left(\boldsymbol{X}_{1: t}\right)\left(r^*\left(\boldsymbol{X}_{1: t}\right)-1\right), 0\right\}}{D^*_{\text {Branch }}\left(p, q \mid \boldsymbol{X}_{1: t-1}\right)}
\end{equation}
\end{tcolorbox}
}

We refer to the above strategy as \emph{Capped Branch Resampling}. It plays a central role in enabling efficient resampling within the hierarchical branch resampling framework. The resampling distribution in \Cref{eq:cap_resample} enables recovery of the full target distribution with only a single resampling step for branches with negative asymmetry. The remaining positions can then be directly sampled from the target model, aligning with the start of the next speculative decoding step and thus incurring no extra computational cost.

We briefly clarify the core mechanism by which capping preserves the target joint distribution. 
From~\Cref{def:r_star} and~\Cref{def:d_star}, it follows that $D^*_{\text{Branch}}(p, q \mid \boldsymbol{X}_{1: t}) =\sum_{\boldsymbol{X}_{1:t}\in \text{Branch}(\boldsymbol{X}_{1:t-1})} \text{max}\{q\left(\boldsymbol{X}_{1:m(\boldsymbol{X}_{1:t})}\right)p\left(\boldsymbol{X}_{m(\boldsymbol{X}_{1:t})+1:t}\right) - q\left(\boldsymbol{X}_{1:t}\right), 0\} $.
Through the acceptance probability and resampling probability at position~$t$, we essentially guarantee that the probability of obtaining $\boldsymbol{X}_{1:t}$ is equal to $q\left(\boldsymbol{X}_{1:m(\boldsymbol{X}_{1:t})}\right)p\left(\boldsymbol{X}_{m(\boldsymbol{X}_{1:t})+1:t}\right)$, partially recovering the probability of the fragment $\boldsymbol{X}_{m(\boldsymbol{X}_{1:t})+1:t}$. And the deficient probability mass $p(\boldsymbol{X}_{1:m(\boldsymbol{X}_{1:t})}) - q(X_{1:m(\boldsymbol{X}_{1:t})})$ is statistically recovered from the resampling distributions in higher hierarchies, which corresponds to the fragments $\tilde{\boldsymbol{X}}_{1:m(\boldsymbol{X}_{1:t})}$ of other trajectories. An illustrative example in~\Cref{app: example_cap} demonstrates how the algorithm recovers loss over the entire path, with a further explanation of the capped ratio provided in \Cref{app:capped_lossless_proof}.

\subsection{Computational Efficiency}
The verification stage in HSD adds negligible overhead compared to the savings from reduced target model forward passes. Thanks to parallelized computations across both the vocabulary and draft positions, HSD is nearly as efficient as tokenwise verification. Runtime measurements (Appendix~\ref{app:computation}) show that verification accounts for less than 1\% of total decoding time, with the majority still spent on draft and target forward passes. These results demonstrate that HSD is not only theoretically lossless but also practically efficient, as further confirmed by our experiments in \Cref{sec:experiment}.

\subsection{Illustrative Example}
\label{sec:hsd_gsm8k_example}

% \noindent\textbf{Illustration Figure of HSD.} In Figure \ref{fig:gsm8k}, we showcase an example from GSM8K of our methods when  we set  $\gamma =7$. Where the draft model is Qwen-2.5 0.5B and the target model is Qwen-2.5 72B.

% \begin{figure*}[!t]
%     \centering
%     \includegraphics[width=1\linewidth]{spd_gsm8k.pdf}
%     \caption{Example of HSD when $\gamma =7$. Each iteration shows the draft model (Qwen-2.5 0.5B), making suggestions that are either accepted (green tokens) or rejected. When rejected, the target model (Qwen-2.5 72B) provides corrections (shown as red and blue tokens). }
%     \label{fig:gsm8k}
% \end{figure*}

\begin{figure}[!t]
\centering
\begin{tcolorbox}[colback=gray!5, colframe=gray!80,
    boxrule=0.5pt, arc=2mm, left=2pt, right=2pt, top=2pt, bottom=2pt,
    title=GSM8K Example, fonttitle=\bfseries]
\scriptsize
\begin{alltt}
Eliza's rate per hour for the first 40 hours she works each week is \$10. She also receives 
an overtime pay of 1.2 times her regular hourly rate. If Eliza worked for 45 hours this week,
how much are her earnings for this week?

To determine Eliza's earnings for the week, we need to calculate both her regular pay and 
her overtime pay.

1. **Calculate Regular Pay:**
   - Eliza's regular rate is \$10 per hour.
   - She worked 40 hours at her regular rate.
   - Regular pay: 40 hours × \$10/hour = \$400.

2. **Calculate Overtime Pay:**
   - Eliza worked a total of 45 hours, so
\end{alltt}
\end{tcolorbox}
\caption{GSM8K question with the generated prefix. The text shown is the printed output of the decoded string in Markdown format.}
\label{fig:gsm8k_example}
\end{figure}

We use a GSM8K question as a running example to demonstrate HSD (see \Cref{fig:gsm8k_example}). This example emphasizes the hierarchical acceptance mechanism and the capping behavior that are key to HSD.

\paragraph{Next-token probabilities.}
Under the given prefix, the large (target) and small (draft) models produce the next-token probabilities:

\begin{tcolorbox}[colback=gray!5, colframe=gray!80,
    boxrule=0.5pt, arc=2mm, left=2pt, right=2pt, top=-6pt, bottom=1pt,
    listing only]
\scriptsize
\begin{align*}
\{p(\cdot\mid \mathbf{X}_{1:\gamma})\} &= 
\{0.7156,1.0000,1.0000,1.0000,0.0000,0.0000,0.0000,1.0000,0.0000,0.4968\},\\
\{q(\cdot\mid \mathbf{X}_{1:\gamma})\} &= 
\{0.8771,0.7900,0.6514,0.2592,0.6773,0.1490,0.5775,1.0000,0.4611,0.3630\}.
\end{align*}
\end{tcolorbox}

The corresponding draft tokens are:
\(\{\texttt{she}, \texttt{work}, \texttt{ed}, \texttt{45}, \texttt{-}, \texttt{40}, \texttt{=}, \texttt{5}, \texttt{hours}, \texttt{of}\}\).

\paragraph{Joint probabilities and ratios.}
We compute the joint probabilities along the draft:

\begin{tcolorbox}[colback=gray!5, colframe=gray!80,
    boxrule=0.5pt, arc=2mm, left=2pt, right=2pt, top=-6pt, bottom=1pt,
    listing only]
\scriptsize
\begin{align*}
\{p(\mathbf{X}_{1:t})\}_{t=1}^{\gamma} &= 
\{0.7156,0.7156,0.7156,0.7156,0,0,0,0,0,0\},\\
\{q(\mathbf{X}_{1:t})\}_{t=1}^{\gamma} &= 
\{0.8771,0.6929,0.4513,0.1170,0.0792,0.0118,0.0068,0.0068,0.0031,0.0011\},\\
\{r(\mathbf{X}_{1:t})\}_{t=1}^{\gamma} &= 
\{0.8159,1.0327,1.5855,6.1171,0,0,0,0,0,0\}.
\end{align*}
\end{tcolorbox}

These ratios exhibit early growth above \(1\) (at \(t=2,3,4\)) and collapse to \(0\) once the target probability vanishes (from \(t \ge 5\)).

\paragraph{Maximum prefix indices and capped ratios.}
Following Definition~\ref{def:max_ratios}, the maximum prefix indices and capped ratios are 
\inlinebox{$\{m(\mathbf{X}_{1:t})\}_{t=1}^{\gamma} = \{0,0,2,3,4,4,4,4,4,4\}$} and
\inlinebox{$\{r^*(\mathbf{X}_{1:t})\}_{t=1}^{\gamma} = \{0.8159,1,1,1,0,0,0,0,0,0\}$}.

\paragraph{Capped branch divergences and acceptance.}
On the full vocabulary branch \(\mathrm{Branch}(\mathbf{X}_{1:t-1})\), we evaluate the capped branch divergences:
\begin{tcolorbox}[colback=gray!5, colframe=gray!80,
    boxrule=0.5pt, arc=2mm, left=2pt, right=2pt, top=-6pt, bottom=1pt,
    listing only]
\scriptsize
\begin{align*}
\{D^*_{\mathrm{Branch}}(p,q\mid \mathbf{X}_{1:t})\}_{t=1}^{\gamma} &= 
\{0.0227,0.2416,0.3343,0.0991,0,0,0,0,0,0\},\\
\{D^*_{\mathrm{Branch}}(q,p\mid \mathbf{X}_{1:t})\}_{t=1}^{\gamma} &=
\{0.1842,0.2416,0.3343,0.0792,0.0991,0.0118,0.0068,0.0068,0.0031,0.0011\}.
\end{align*}
\end{tcolorbox}

The hierarchical acceptance (Eq.~\ref{def:acceptance_probability}) then yields 
\inlinebox{$\{h_t\}_{t=1}^{\gamma} = \{0.1231,1,1,1,0,0,0,0,0,0\}$}.

Acceptance saturates at \(t=2,3,4\), implying the first four tokens are validated, i.e., \(n_{\mathrm{match}}=4\).

\paragraph{Comparison to tokenwise verification.}
For a tokenwise baseline that validates strictly left-to-right, the per-position magnitudes are
\inlinebox{$\{h_t^{\mathrm{tokenwise}}\}_{t=1}^{\gamma} = \{0.8159,1,1,1,0,0,0,1,0,1\}$}.

Since the baseline commits at the first position, an initial \(h_1 = 0.8159\) may trigger rejection and discard the entire draft block. 
% \usepackage[skins, breakable]{tcolorbox} % make sure to include this

% --- short inline box shortcut ---
% \newtcbox{\inlinebox}{on line,
%   colback=gray!5, colframe=gray!80, boxrule=0.5pt, arc=2mm,
%   left=1pt, right=1pt, top=0pt, bottom=0pt, boxsep=1pt, 
%   fontupper=\scriptsize}

\section{Experiments}
\label{sec:experiment}

In this section, we empirically demonstrate the superiority of HSD with comparison on various benchmarks and configurations, comprehensive ablation studies, and in-depth analysis of results.

\subsection{Experiment Setting}

\noindent \textbf{Experiments Setup.} Experiments are conducted with the widely adopted GPTQ-quantized 8-bit instruction-tuned Qwen2.5 series \citep{bai2023qwen}. By default, we employ the 0.5B as the draft model and 72B as the target models, with a temperature of 1. We leverage GSM8K~\citep{cobbe2021trainingverifierssolvemath} for mathematical problem-solving, HumanEval~\citep{chen2021evaluatinglargelanguagemodels} for code generation, and CNN/DailyMail~\citep{see-etal-2017-get} for text summarization. All experiments were conducted on a single NVIDIA H20 GPU with 96 GB of memory, unless otherwise specified.

\noindent \textbf{Baselines and Metrics.} 
We compare two lossless verification methods—Token-wise and Block-wise—using two metrics: \textit{Block Efficiency} (tokens/step) and \textit{Decoding Speed} (tokens/second). \emph{Block Efficiency} measures the average tokens generated per serial call to the target model, reflecting intrinsic efficiency independent of hardware. \emph{Decoding Speed} indicates tokens produced per second for practical reference. Additional details and extended evaluations are in \Cref{ap:experiment_extenstion}.

\subsection{Experiment Results}

\begin{table*}[!b]
\centering
\caption{Comparison of Block Efficiency (BE) and Decoding Speed (DS) across datasets and model scales. Values in parentheses show percentage improvement over Tokenwise.}
\label{tab:main_table}
\scriptsize
\renewcommand{\arraystretch}{1.1}
\begin{tabular}{@{}l|ccc|ccc@{}}
\toprule[1.2pt]
Method & \multicolumn{3}{c|}{Block Efficiency (Token/Step)} & \multicolumn{3}{c}{Decoding Speed (Token/Second)} \\
& 14B & 32B & 72B & 14B & 32B & 72B \\
\midrule
\multicolumn{7}{c}{\textbf{GSM8K}} \\
\hline
Tokenwise & 5.99 & 6.14 & 6.44 & 82.28 & 53.87 & 31.49 \\
Blockwise & 6.13 (+2.3\%) & 6.26 (+2.0\%) & 6.53 (+1.4\%) & 86.06 (+4.6\%) & 54.91 (+1.9\%) & 31.79 (+1.0\%) \\
\textbf{HSD (Ours)} & \textbf{6.30 (+5.2\%)} & \textbf{6.47 (+5.4\%)} & \textbf{6.65 (+3.3\%)} & \textbf{91.05 (+10.7\%)} & \textbf{57.12 (+6.0\%)} & \textbf{32.52 (+3.3\%)} \\
\midrule
\multicolumn{7}{c}{\textbf{HumanEval}} \\
\hline
Tokenwise & 4.83 & 4.89 & 5.23 & 74.21 & 45.68 & 26.31 \\
Blockwise & 5.11 (+5.8\%) & 5.15 (+5.3\%) & 5.34 (+2.1\%) & 78.14 (+5.3\%) & 48.15 (+5.4\%) & 26.96 (+2.5\%) \\
\textbf{HSD (Ours)} & \textbf{5.29 (+9.5\%)} & \textbf{5.49 (+12.3\%)} & \textbf{5.40 (+3.3\%)} & \textbf{81.09 (+9.3\%)} & \textbf{50.88 (+11.4\%)} & \textbf{27.48 (+4.4\%)} \\
\midrule
\multicolumn{7}{c}{\textbf{CNN/DailyMail}} \\
\hline
Tokenwise & 2.39 & 2.36 & 2.35 & 37.28 & 21.89 & 11.90 \\
Blockwise & 2.50 (+4.6\%) & 2.42 (+2.5\%) & 2.39 (+1.7\%) & 38.54 (+3.4\%) & 22.31 (+1.9\%) & 12.10 (+1.4\%) \\
\textbf{HSD (Ours)} & \textbf{2.59 (+8.4\%)} & \textbf{2.46 (+4.2\%)} & \textbf{2.45 (+4.3\%)} & \textbf{39.96 (+7.2\%)} & \textbf{22.78 (+4.1\%)} & \textbf{12.33 (+3.6\%)} \\
\bottomrule[1.2pt]
\end{tabular}
\end{table*}

% \begin{table}[!ht]
%     \centering
%     \scriptsize
%     \caption{\small{Comparison in Multi-draft setting. Hierarchical is compared to Tokenwise, and Hierarchical Multi-draft is compared to Tokenwise Multi-draft.}}
%     \label{tab:multi}
%     \begin{tabular}{@{}l|ccc@{}}
%     \toprule
%     \multirow{2}*{Method} & \multicolumn{3}{c}{Block Efficiency (Token/Step)} \\
%     & GSM8K & HumanEval & CNN/DailyMail \\
%     \midrule
%     Tokenwise & 6.44 & 5.23 & 2.35 \\
%     \textbf{Hierarchical (Ours)} & \textbf{6.65 (+3.3\%)} & \textbf{5.40 (+3.3\%)} & \textbf{2.45 (+4.3\%)} \\
%     \midrule
%     Tokenwise Multi-draft & 8.65 & 7.96 & 3.79 \\
%     \textbf{Hierarchical Multi-draft (Ours)} & \textbf{8.89 (+2.8\%)} & \textbf{8.26 (+3.8\%)} & \textbf{4.21 (+11.1\%)} \\
%     \midrule[\heavyrulewidth]
%     \multirow{2}*{Method} & \multicolumn{3}{c}{Decoding Speed (Token/Second)} \\
%     & GSM8K & HumanEval & CNN/DailyMail \\
%     \midrule
%     Tokenwise & 31.49 & 26.31 & 11.90 \\
%     \textbf{Hierarchical (Ours)} & \textbf{32.52 (+3.3\%)} & \textbf{27.48 (+4.4\%)} & \textbf{12.33 (+3.6\%)} \\
%     \midrule
%     Tokenwise Multi-draft & 37.66 & 35.72 & 15.38 \\
%     \textbf{Hierarchical Multi-draft (Ours)} & \textbf{38.41 (+2.0\%)} & \textbf{36.83 (+3.1\%)} & \textbf{16.75 (+8.9\%)} \\
%     \bottomrule
%     \end{tabular}
% \end{table}

\begin{table}[!ht]
    \centering
    \caption{Comparison of our HSD and tokenwise verification in Multi-draft setting.}
    \label{tab:multi}
        \scriptsize
    \begin{tabular}{@{}l|ccc|ccc@{}}
    \toprule
    \multirow{2}*{Method} & \multicolumn{3}{c|}{Block Efficiency (Token/Step)} & \multicolumn{3}{c}{Decoding Speed (Token/Second)} \\
    & GSM8K & HumanEval & CNN/DailyMail & GSM8K & HumanEval & CNN/DailyMail \\
    \midrule
    Tokenwise & 6.44 & 5.23 & 2.35 & 31.49 & 26.31 & 11.90 \\
    \textbf{HSD (Ours)} & \textbf{6.65 (+3.3\%)} & \textbf{5.40 (+3.3\%)} & \textbf{2.45 (+4.3\%)} & \textbf{32.52 (+3.3\%)} & \textbf{27.48 (+4.4\%)} & \textbf{12.33 (+3.6\%)} \\
    \midrule
    Tokenwise Multi-draft & 8.65 & 7.96 & 3.79 & 37.66 & 35.72 & 15.38 \\
    \textbf{HSD Multi-draft (Ours)} & \textbf{8.89 (+2.8\%)} & \textbf{8.26 (+3.8\%)} & \textbf{4.21 (+11.1\%)} & \textbf{38.41 (+2.0\%)} & \textbf{36.83 (+3.1\%)} & \textbf{16.75 (+8.9\%)} \\
    \bottomrule
    \end{tabular}
\end{table}

\noindent \textbf{Main results.}~
Table~\ref{tab:main_table} summarizes the performance of HSD across datasets and model scales using the Qwen2.5 suite (0.5B as draft,14B, 32B, and 72B as targets).
Overall, HSD consistently improves both Block Efficiency (BE) and Decoding Speed (DS) relative to Tokenwise and Blockwise verification. For \textbf{GSM8K}, the gains are stable across scales, with BE improvements of \textbf{5.2\%--5.4\%} at 14B/32B and \textbf{3.3\%} at 72B, accompanied by DS increases of up to \textbf{10.7\%}. 
On \textbf{HumanEval}, the effect is more pronounced: BE rises by \textbf{9.5\%} and \textbf{12.3\%} at 14B and 32B, while DS improves by \textbf{9.3\%} and \textbf{11.4\%}; even at 72B, HSD maintains positive margins (\textbf{3.3\%} BE, \textbf{4.5\%} DS). 
For \textbf{CNN/DailyMail}, the improvements are moderate but consistent, with BE gains of \textbf{4.2\%--8.4\%} and DS gains of \textbf{3.4\%--7.2\%}. On average, HSD provides consistent advantages over Tokenwise and Blockwise verification, with improvements of approximately \textbf{6.2\% in BE} and \textbf{6.7\% in DS}.

\noindent \textbf{Multi-draft.}~To demonstrate the compatibility of HSD, we compare it with token-wise verificaiton in a multi-draft setting. For simplicity---and without loss of generality---we adopt Recursive Reject Sampling (RRS) with replacement~\citep{yang2024multi} as the baseline for its scalability and independence from complex tree attention mechanisms. Notably, since it is not straightforward to extend blockwise verification to the multi-draft setup, we omit it from our comparison. We evaluated multi-draft generation with 11 candidate drafts in Table~\ref{tab:multi}, and HSD yields an average 5.9\% improvement in Block Efficiency and 4.7\% improvement in Decoding Speed over token-wise decoding. 

\newcommand{\subtablewidth}{0.49\textwidth}
\begin{table*}[!t]
\caption{Ablations on temperature, draft length, and target model size on GSM8K. Except for the ablation on target model size, we adopt Qwen2.5-0.5B and Qwen2.5-72B as the draft and target pair.}
\label{tab:ablation}
\setlength{\abovecaptionskip}{0pt} 
\setlength{\belowcaptionskip}{0pt} 
\centering
% \vspace{-0.2cm}
\setlength{\tabcolsep}{0.03cm}
\hfill
\begin{subtable}[t]{\subtablewidth}
    \centering
    \scriptsize
        \caption{\small{Ablation on \textbf{temperature} ($\gamma=10$).}} 
        \setlength{\abovecaptionskip}{0pt} 
\setlength{\belowcaptionskip}{2pt} % Caption for the first table
        \label{tab:target_model_size} % A more descriptive label
    \resizebox{\linewidth}{!}{
        \centering
            \begin{tabular}{@{}l|cccccccccccc@{}}
    \toprule
    \multirow{2}*{Method} & \multicolumn{3}{c}{Block Efficiency} & \multicolumn{3}{c}{Decoding Speed}      \\
    &  $t =0.6$ & $t=0.8$ & $t=1$ & $t=0.6$ & $t=0.8$ & $t=1$ \\
    \midrule
   Tokenwise  &6.81 & 6.70& 6.44&32.86&32.18 &31.49\\
    Blockwise  & 6.83& 6.74& 6.53&33.07&32.33 &31.79\\
    Hierarchicial  & 6.86 & 6.79 & 6.65&33.21&32.90 &32.52 \\
    \bottomrule
    \end{tabular}
}
\end{subtable}
\hfill
\begin{subtable}[t]{\subtablewidth}
    \centering
    \scriptsize
        \caption{\small{Ablation on \textbf{draft lengths} ($t=1$).}} %
        \setlength{\abovecaptionskip}{0pt} 
\setlength{\belowcaptionskip}{2pt} 
    \resizebox{1\linewidth}{!}{
        \begin{tabular}{@{}l|cccccc@{}} % Adjusted column spec
            \toprule
            \multirow{2}*{Method} & \multicolumn{3}{c}{Block Efficiency} & \multicolumn{3}{c}{Decoding Speed} \\
            & $\gamma=5$ & $\gamma=10$ & $\gamma=15$ & $\gamma=5$ & $\gamma=10$ & $\gamma=15$ \\
            \midrule
            Tokenwise  &4.48& 6.44&7.61 & 12.01&31.49&51.03 \\
            Blockwise &4.52 &6.53 & 7.74& 12.14&31.79&51.75\\
            Hierarchical &4.59 & 6.65& 7.88& 12.35& 32.52&52.95\\
            \bottomrule
        \end{tabular}}
\end{subtable}
% \vspace{-0.3cm}
\end{table*}

\newcommand{\subtablewidthA}{0.58\textwidth}
\newcommand{\subtablewidthB}{0.41\textwidth}

\begin{table*}[!t]
\caption{Extended experimental results using the LLaMA model family and EAGLE-3 framework on GSM8K. Note that we replace EAGLE-3’s tokenwise verification with our HSD, yielding EAGLE-3H.}
\label{tab:extended}
\setlength{\abovecaptionskip}{0pt}
\setlength{\belowcaptionskip}{0pt}
\centering
\scriptsize
\setlength{\tabcolsep}{0.12cm}
% ---------- First subtable ----------
\begin{subtable}[t]{\subtablewidthA}
\centering
\caption{Evaluation using the LLaMA-3 model family.}
\begin{tabular}{@{}lcccc@{}}
\toprule
\multirow{2}{*}{Method} 
& \multicolumn{2}{c}{Single-draft} 
& \multicolumn{2}{c}{Multi-draft} \\
\cmidrule(lr){2-3} \cmidrule(lr){4-5}
& Block Eff. & Decoding Speed 
& Block Eff. & Decoding Speed \\
\midrule
Tokenwise  & 6.83 & 8.41  & 8.72 & 10.21 \\
Blockwise  & 7.32($+7.2\%$) & 8.87(+5.5\%)  & N/A  & N/A  \\
\textbf{HSD (Ours)} & \textbf{7.43(+8.8\%)} & \textbf{9.18(+9.2\%)} 
            & \textbf{9.00(+3.2\%)} & \textbf{11.02(+7.9\%)} \\
\bottomrule
\end{tabular}
\label{tab:llama-performance}
\end{subtable}
\hfill
% ---------- Second subtable ----------
\begin{subtable}[t]{\subtablewidthB}
\centering
\scriptsize
\caption{Integration with EAGLE-3.}
\begin{tabular}{@{}lcc@{}}
\toprule
Method & Block Eff. & Decoding Speed \\
\midrule
EAGLE-3         & 3.40 & 71.59  \\
Blockwise       & N/A   & N/A  \\
\textbf{EAGLE-3H (Ours)} & \textbf{3.55(+4.4\%)} & \textbf{80.49(+12.4\%)}  \\
\bottomrule
\end{tabular}
\label{tab:Eagle_integration}
\end{subtable}
\end{table*}

% \vspace{-1cm}

% \subsection{Ablation Study}
\noindent \textbf{Ablation on Temperature.}~We conduct a systematic evaluation of sampling temperature's effect on decoding efficiency, with $t \in \{0.6, 0.8, 1.0\}$ (Table~\ref{tab:ablation}(a)). HSD consistently outperforms other approaches across all temperature settings,  demonstrating its robustness to temperature variations.

% \noindent \textbf{Analysis of Draft Length Impact}
% \noindent \textbf{Ablation on Draft Length.}~We evaluate draft lengths $\gamma \in \{5, 10, 15\}$ tokens, where HSD consistently outperforms baselines with increasing efficiency gains (Table~\ref{tab:ablation}(c)). At $\gamma=15$, HSD achieves peak performance with 7.88 tokens/step in block efficiency and 52.95 tokens/second in decoding speed, representing improvements of 4.29\% and 1.82\% over Blockwise, respectively. The consistent performance advantage across all draft lengths demonstrates HSD's robust scalability.

\noindent \textbf{Ablation on Draft Length.}~We evaluate draft lengths $\gamma \in \{5, 10, 15\}$ tokens, where HSD consistently outperforms baselines with increasing efficiency gains (Table~\ref{tab:ablation}(b)). At $\gamma=15$, HSD achieves peak performance with 7.88 tokens/step in block efficiency and 52.95 steps/second in decoding speed, representing improvements of 3.58\% and 3.88\% over Tokenwise, respectively. The consistent performance advantage across all draft lengths demonstrates HSD's robust scalability.

\noindent\textbf{Extended Results.}
% \noindent\textbf{Experiments on Llama}
We conducted additional experiments using Llama-3.1-70B-Instruct and Llama-3.1-8B-Instruct pair (non-quantized version), with model weights distributed on 8 H20 GPUs. The results are shown in \Cref{tab:llama-performance}.
Moreover, we integrated HSD into the SOTA EAGLE-3-LLaMa3.1-Instruct-8B ($\gamma=7$) by replacing its tokenwise verifier in Table \ref{tab:Eagle_integration}. Following EAGLE-3, we accept at least the first draft token for a fair comparison. Note that EAGLE-3 utilizes top-K sampling for drafting, making all draft probabilities equal to 1. In this case, any verification method theoretically degenerates into the same behavior and the observed gain in block efficiency of HSD is likely influenced by sampling stochasticity and floating-point precision. However, the observed significant practical speedup in decoding speed is expected, since our implementation (see \Cref{ap:python}) avoids the explicit loops in EAGLE's implementation of tokenwise verification.

\section{Conclusion}
We present HSD, a lossless verification method that maximizes accepted tokens while provably preserving the full target distribution. Supported by theoretical guarantees and extensive experiments, HSD consistently accelerates inference across models and benchmarks. Its drop-in integration with frameworks like EAGLE-3 demonstrates both practicality and broad applicability. HSD sets a new standard for efficient, lossless speculative decoding in large language models.

\bibliographystyle{plainnat}
\bibliography{reference}

\appendix

\section*{Appendix}
\addcontentsline{toc}{section}{Appendix}  % optional: adds appendix to TOC

\renewcommand{\thesubsection}{\Alph{subsection}}  % makes subsection labels A, B, ...
\renewcommand{\theequation}{A.\arabic{equation}} 
\renewcommand{\thefigure}{A.\arabic{figure}} 
\renewcommand{\thetable}{A.\arabic{table}} 
\setcounter{equation}{0}
\setcounter{figure}{0}
\setcounter{table}{0}
% OR use \section if you want Appendix A, B, etc.
% e.g., \section{Proof of Theorem 1}
\subsection{Theoretical Foundation}

\subsubsection{Symmetry of Total Divergence}
\label{app:symmetry_TD}

\begin{lemma}
\label{lem:symmetry}
\textbf{Symmetry of Total Divergence.} 
\begin{equation}
\setlength{\abovedisplayskip}{2pt}
\setlength{\belowdisplayskip}{2pt}
D_{\Omega}(p, q) = D_{\Omega}(q, p). 
\end{equation}

\end{lemma}

\begin{proof}
\label{proof:total}
From \Cref{def:general}, we know:
\begin{equation}
\begin{aligned}
D_\Omega(p, q) - D_\Omega(q, p)
&= \sum_{\tilde{\omega} \in \Omega} \max\{p(\tilde{\omega}) - q(\tilde{\omega}), 0\} - \sum_{\tilde{\omega} \in \Omega} \max\{q(\tilde{\omega}) - p(\tilde{\omega}), 0\} \\
&= \sum_{\substack{\tilde{\omega} \in \Omega \\ p(\tilde{\omega}) \geq q(\tilde{\omega})}} (p(\tilde{\omega}) - q(\tilde{\omega})) - \sum_{\substack{\tilde{\omega} \in \Omega \\ q(\tilde{\omega}) > p(\tilde{\omega})}} (q(\tilde{\omega}) - p(\tilde{\omega})) \\
&= \sum_{\tilde{\omega} \in \Omega} p(\tilde{\omega}) - \sum_{\tilde{\omega} \in \Omega} q(\tilde{\omega}) \\
&= 0 \quad \text{(since both } p \text{ and } q \text{ sum to } 1 \text{ over the full sample space } \Omega)
\end{aligned}
\end{equation}
Thus, \( D_\Omega(p, q) = D_\Omega(q, p) \), completing the proof.
\end{proof}

\subsubsection{Partial Distribution Recovery}
\label{app:proof_partial_recovery}

\begin{proof}[Proof of Theorem~\ref{theorem:derivation}]

Let \(P(w \text{ is yielded})\) denote the total probability of producing \(w \in \Omega'\). By construction, this can be decomposed as
\begin{equation}
P(w \text{ is yielded}) = P(w \text{ is drafted \& accepted}) + P(w \text{ is drafted \& rejected, } w \text{ is resampled}),
\end{equation}

where acceptance occurs with probability \(h(w) = \min\{p(w)/q(w),1\}\), and resampling follows the distribution \(P_{\text{res}}(\cdot \mid \Omega')\) with total trigger mass \(D_{\Omega'}(q,p)\). Here, the total trigger mass represents the sum of probabilities of all draft outcomes in \(\Omega'\) that are rejected. Hence,

\begin{equation}
P(w \text{ is yielded}) = h(w)\, q(w) + D_{\Omega'}(q,p) \, P_{\text{res}}(w \mid \Omega').
\end{equation}

Noting that \(h(w) \, q(w) = \min\{p(w), q(w)\}\), we have

\begin{equation}
P(w \text{ is yielded}) = \min\{p(w), q(w)\} + D_{\Omega'}(q,p) \, P_{\text{res}}(w \mid \Omega').
\end{equation}

To match the target distribution exactly (\(P(w \text{ is yielded}) = p(w)\)), we require
\begin{equation}
P_{\text{res}}(w \mid \Omega') = \frac{p(w) - \min\{p(w), q(w)\}}{D_{\Omega'}(q,p)} = \frac{\max\{p(w)-q(w), 0\}}{D_{\Omega'}(q,p)}.
\end{equation}

Summing over all \(w \in \Omega'\) gives
\begin{equation}
\sum_{w \in \Omega'} P_{\text{res}}(w \mid \Omega') = \frac{D_{\Omega'}(p,q)}{D_{\Omega'}(q,p)}.
\end{equation}

For \(P_{\text{res}}(\cdot \mid \Omega')\) to be a valid probability distribution, this sum must not exceed 1. Therefore, the necessary and sufficient condition is
\begin{equation}
D_{\Omega'}(p,q) \le D_{\Omega'}(q,p),
\end{equation}
which completes the proof.
\end{proof}

\subsubsection{Quantification Analysis of Asymmetry}
\label{proof:quantify}
\begin{proof}
From \Cref{def:asym} and \Cref{def:branch_divergence}, we obtain:
\begin{equation}
\begin{aligned}
\Delta_\text{Branch}(\boldsymbol{X}_{1:t-1})
&= \sum_{\boldsymbol{X}_{1:t} \in \text{Branch}(\boldsymbol{X}_{1:t-1})} \max\left\{ p\left(\boldsymbol{X}_{1:t}\right) - q\left(\boldsymbol{X}_{1:t}\right), 0 \right\} \\
&\quad - \sum_{\boldsymbol{X}_{1:t} \in \text{Branch}(\boldsymbol{X}_{1:t-1})} \max\left\{ q\left(\boldsymbol{X}_{1:t}\right) - p\left(\boldsymbol{X}_{1:t}\right), 0 \right\} \\
&= \sum_{\boldsymbol{X}_{1:t} \in \text{Branch}(\boldsymbol{X}_{1:t-1})} p\left(\boldsymbol{X}_{1:t}\right) - \sum_{\boldsymbol{X}_{1:t} \in \text{Branch}(\boldsymbol{X}_{1:t-1})} q\left(\boldsymbol{X}_{1:t}\right) \\
&= \sum_{x_t \in \mathcal{V}} p\left(\boldsymbol{X}_{1:t-1}\right) p\left(x_t \mid \boldsymbol{X}_{1:t-1}\right) - \sum_{x_t \in \mathcal{V}} q\left(\boldsymbol{X}_{1:t-1}\right) q\left(x_t \mid \boldsymbol{X}_{1:t-1}\right) \\
&= p\left(\boldsymbol{X}_{1:t-1}\right) - q\left(\boldsymbol{X}_{1:t-1}\right) \quad \text{(since } \sum_{x_t \in \mathcal{V}} p(x_t \mid \boldsymbol{X}_{1:t-1}) = 1)
\end{aligned}
\end{equation}
\end{proof}

\subsubsection{Relation to the Divergence in \cite{leviathan2023fast}}
\label{sec:equivalence}
\begin{lemma}
\label{lem:equivalence}
The total divergence is equivalent to the divergence defined in \cite{leviathan2023fast} for token distributions over the full sample space.
\end{lemma}

\begin{proof}
Following \cite{leviathan2023fast}, let $\tilde{x}$ denote a token, and omit conditions in the token probabilities for simplicity.  
From Definition 3.2 in \cite{leviathan2023fast}, we have:
\begin{equation}
\begin{aligned}
D_{\text{LK}}(p, q) 
&= \sum_{\tilde{x} \in \Omega} \left| \frac{p(\tilde{x}) - q(\tilde{x})}{2} \right| \\
&= \frac{1}{2} \left( \sum_{\tilde{x} \in \Omega} \max\{p(\tilde{x}) - q(\tilde{x}), 0\} + \sum_{\tilde{x} \in \Omega} \max\{q(\tilde{x}) - p(\tilde{x}), 0\} \right)
\end{aligned}
\end{equation}

From \Cref{lem:symmetry}, we know that \( D_\Omega(p, q) = D_\Omega(q, p) \), so we can write:
\begin{equation}
\begin{aligned}
D_\Omega(p, q) 
&= \frac{D_\Omega(p, q) + D_\Omega(q, p)}{2} \\
&= \frac{1}{2} \left( \sum_{\tilde{x} \in \Omega} \max\{p(\tilde{x}) - q(\tilde{x}), 0\} + \sum_{\tilde{x} \in \Omega} \max\{q(\tilde{x}) - p(\tilde{x}), 0\} \right)
\end{aligned}
\end{equation}

Therefore, \( D_\Omega(p, q) = D_{\text{LK}}(p, q) \), completing the proof.
\end{proof}

\subsubsection{Hierarchy of Divergence}

\begin{proof}
\label{proof:hierarchy}
Proof of \Cref{thm:hierarchy}. 

From \Cref{theorem:branch}, we recall that:
\begin{equation}
\Delta_\text{Branch}(\boldsymbol{X}_{1:t-2}, \tilde{x}_{t-1}) = p(\boldsymbol{X}_{1:t-2}, \tilde{x}_{t-1}) - q(\boldsymbol{X}_{1:t-2}, \tilde{x}_{t-1}).
\end{equation}

Therefore, summing over the cases where this difference is positive gives:
\begin{equation}
\sum\limits_{\Delta_\text{Branch}(\boldsymbol{X}_{1:t-2}, \tilde{x}_{t-1}) > 0} \Delta_\text{Branch}(\boldsymbol{X}_{1:t-2}, \tilde{x}_{t-1}) 
= \sum_{\tilde{x}_{t-1} \in \mathcal{V}} \max \left\{ p(\boldsymbol{X}_{1:t-2}, \tilde{x}_{t-1}) - q(\boldsymbol{X}_{1:t-2}, \tilde{x}_{t-1}), 0 \right\}.
\end{equation}

By Definition~\ref{def:branch_divergence}, this is precisely the branch divergence one level higher $D_\text{Branch}(p, q \mid \boldsymbol{X}_{1:t-2})$, thus completing the proof.
\end{proof}

\subsection{Lossless of Naive Hierarchical Speculative Decoding}
\label{app:naive-hsd}

\subsubsection{Illustrative Example}
\label{proof:naive-example}
For example, consider the case where $r(\boldsymbol{X}_{1:\gamma})>1$, $r(\boldsymbol{X}_{1:\gamma-1})>1$, and $r(\boldsymbol{X}_{1:\gamma-2})\le 1$. The accept term is simply equal to $q(\boldsymbol{X}_{1:\gamma})$, so we only need to check whether the resampling term equals $p(\boldsymbol{X}_{1:\gamma}) - q(\boldsymbol{X}_{1:\gamma})$. According to \Cref{eq:resample_prob_naive}, we know $P_\text{res}(x_{\gamma-2}\mid \boldsymbol{X}_{1:\gamma-1})=0$. Consequently, contributions from positions earlier than $\gamma-1$ in the sum above vanish, which implies that the resampling term for $\boldsymbol{X}_{1:\gamma}$ arises solely from resampling at positions $\gamma$ and $\gamma-1$ as follows:
\begin{equation}
\begin{aligned}
&\sum_{\tilde{x}_{\gamma}} P\bigl(\text{sample }\boldsymbol{X}_{1:\gamma-1}\tilde{x}_{\gamma},\text{reject } \tilde{x}_{\gamma},\text{accept }\boldsymbol{X}_{1:\gamma-1},\text{resample }x_{\gamma} \bigr) +\\
& \sum_{\tilde{\boldsymbol{X}}_{\gamma-1:\gamma}} P\bigl(\text{sample }\boldsymbol{X}_{1:\gamma-2}\tilde{x}_{\gamma-1:\gamma} , \text{reject  }\tilde{\boldsymbol{X}}_{\gamma-1:\gamma} ,\text{accept }\boldsymbol{X}_{1:\gamma-2} ,\text{resample }\boldsymbol{X}_{\gamma-1:\gamma}   \bigr)\\
=& \sum_{\tilde{x}_{\gamma}}\underbrace{q(\boldsymbol{X}_{1:\gamma-1}\tilde{x}_{\gamma})}_{\text{draft probability}}\cdot \underbrace{(1-h_\gamma)}_{\text{reject backwards at $\tau+1=\gamma$}} \cdot \underbrace{h_\gamma}_{\text{accept $\boldsymbol{X}_{1:\gamma-1}$}} \cdot \underbrace{P_\text{res}(x_{t})}_{\text{resample at $\tau+1=\gamma$}} + \\
&\sum_{\tilde{x}_{\gamma-1}}\sum_{\tilde{x}_{\gamma}}\underbrace{q(\boldsymbol{X}_{1:\gamma-2}\tilde{x}_{\gamma-1}\tilde{x}_{\gamma})}_{\text{draft probability}}\!\cdot\! \underbrace{(1-h_\gamma)(1-h_{\gamma-1})}_{\text{reject backwards at $\tau\!+\!1\!=\!\gamma\!-\!1$}} \!\cdot\! \underbrace{h_{\gamma-2}}_{\text{accept $\boldsymbol{X}_{1:\gamma-1}$}} \!\cdot\!\underbrace{P_\text{res}(x_{\gamma-1})P_\text{res}(x_{\gamma})}_{\text{resample at $\tau+1=\gamma$}}
\end{aligned}
\end{equation}
From \Cref{def:branch_divergence} that the excess probability mass that triggers resampling
$D_{\mathrm{Branch}}(q,p\mid \boldsymbol{X}_{1:\gamma-1})
=\sum_{\tilde{x}_{\gamma}} q(\boldsymbol{X}_{1:\gamma-1}\tilde{x}_{\gamma}) (1-h_\gamma)$. Then we have:
\begin{equation}
    \begin{aligned}
        =& D_\text{Branch}(q, p|\boldsymbol{X}_{1:\gamma-1})\cdot 1 \cdot \frac{p(\boldsymbol{X}_{1:\gamma}) - q(\boldsymbol{X}_{1:\gamma})}{D_\text{Branch}(p, q|\boldsymbol{X}_{1:\gamma-1})} + \\
        &\sum_{\tilde{x}_{\gamma-1}}D_\text{Branch}(q, p|\boldsymbol{X}_{1:\gamma-2}\tilde{x}_{\gamma-1})(1-\frac{D_\text{Branch}(p, q|\boldsymbol{X}_{1:\gamma-2}\tilde{x}_{\gamma-1})}{D_\text{Branch}(q, p|\boldsymbol{X}_{1:\gamma-2}\tilde{x}_{\gamma-1})})P_\text{res}(x_{\gamma-1})P_\text{res}(x_{\gamma})
  % \frac{\text{max}\{p(x^\gamma) - q(x^\gamma), 0\}}{D_\text{Branch}(p, q|X_{1:\gamma-1})}
   \end{aligned}
\end{equation}
From \Cref{def:asym} and \Cref{thm:hierarchy}, we know that $\sum_{\tilde{x}_{\gamma-1}} D_\text{Branch}(q, p|\boldsymbol{X}_{1:\gamma-2}\tilde{x}_{\gamma-1}) - D_\text{Branch}(p, q|\boldsymbol{X}_{1:\gamma-2}\tilde{x}_{\gamma-1}) = D_\text{Branch}(q, p|\boldsymbol{X}_{1:\gamma-2})$. Then we have:
\begin{equation}
\begin{aligned}
= &\frac{D_\text{Branch}(q, p|\boldsymbol{X}_{1:\gamma-1})}{D_\text{Branch}(p, q|\boldsymbol{X}_{1:\gamma-1})} \cdot (p(\boldsymbol{X}_{1:\gamma}) - q(\boldsymbol{X}_{1:\gamma})) +\\
& D_\text{Branch}(q, p|\boldsymbol{X}_{1:\gamma-2}) \cdot \frac{p(\boldsymbol{X}_{1:\gamma-1}) - q(\boldsymbol{X}_{1:\gamma-1})}{D_\text{Branch}(p, q|\boldsymbol{X}_{1:\gamma-2})} \cdot \frac{p(\boldsymbol{X}_{1:\gamma}) - q(\boldsymbol{X}_{1:\gamma})}{D_\text{Branch}(p, q|\boldsymbol{X}_{1:\gamma-1})}
\end{aligned}
\end{equation}
   
We know from \Cref{def:asym} and \Cref{theorem:branch} that $p(\boldsymbol{X}_{1:\gamma}) - q(\boldsymbol{X}_{1:\gamma})=D_\text{Branch}(p, q|\boldsymbol{X}_{1:\gamma-1}) - D_\text{Branch}\bigl(q, p|\boldsymbol{X}_{1:\gamma-1}\bigr)$. Then we have:
\begin{equation}
  \begin{aligned}  
     =& \frac{D_\text{Branch}(q, p|\boldsymbol{X}_{1:\gamma-1})}{D_\text{Branch}(p, q|\boldsymbol{X}_{1:\gamma-1})} \cdot (p(\boldsymbol{X}_{1:\gamma}) - q(\boldsymbol{X}_{1:\gamma})) + \\
&\frac{\bigl(D_\text{Branch}(p,q \mid \boldsymbol{X}_{1:\gamma-1}) - D_\text{Branch}(q,p \mid \boldsymbol{X}_{1:\gamma-1})\bigr)}{D_\text{Branch}(p,q \mid \boldsymbol{X}_{1:\gamma-1})} \cdot (p(\boldsymbol{X}_{1:\gamma}) - q(\boldsymbol{X}_{1:\gamma}))\\
&=p(\boldsymbol{X}_{1:\gamma}) - q(\boldsymbol{X}_{1:\gamma}) 
\end{aligned}
\end{equation}

\begin{align}
        &\sum_{\tilde{x}_{\gamma}} P\bigl(\boldsymbol{X}_{1:\gamma-1}\tilde{x}_{\gamma} \; \text{is sampled}, \tilde{x}_{\gamma}  \; \text{is rejected},\boldsymbol{X}_{1:\gamma-1}  \; \text{is accepted},x_{\gamma}  \; \text{is resampled} \bigr) + \\
        & \sum_{\boldsymbol{X}'_{\gamma-1:\gamma}} P\bigl(\boldsymbol{X}_{1:\gamma-2}\tilde{\boldsymbol{X}}_{\gamma-1:\gamma} \text{is sampled}, \tilde{\boldsymbol{X}}_{\gamma-1:\gamma}  \text{is rejected},\boldsymbol{X}_{1:\gamma-2}   \text{is accepted},\boldsymbol{X}_{\gamma-1:\gamma}   \text{is resampled} \bigr) \\
        &= \sum_{\tilde{x}_{\gamma}}
\underbrace{q(\boldsymbol{X}_{1:\gamma-1}\tilde{x}_{\gamma})}_{\text{draft probability}}\cdot \underbrace{(1-h_\gamma)}_{\text{reject backwards at $\tau+1=\gamma$}} \cdot \underbrace{h_\gamma}_{\text{accept $\boldsymbol{X}_{1:\gamma-1}$}} \cdot \underbrace{P_\text{res}(x_{t})}_{\text{resample at $\tau+1=\gamma$}} + \\
       &\sum_{\tilde{x}_{\gamma-1}}\sum_{\tilde{x}_{\gamma}}
\underbrace{q(\boldsymbol{X}_{1:\gamma-2}\tilde{x}_{\gamma-1}\tilde{x}_{\gamma})}_{\text{draft probability}}\cdot \underbrace{(1-h_\gamma)(1-h_{\gamma-1})}_{\text{reject backwards at $\tau+1=\gamma-1$}} \cdot \underbrace{h_{\gamma-2}}_{\text{accept $\boldsymbol{X}_{1:\gamma-1}$}} \cdot \underbrace{P_\text{res}(x_{\gamma-1})P_\text{res}(x_{\gamma})}_{\text{resample at $\tau+1=\gamma$}} \\
\intertext{From \Cref{def:branch_divergence} that the excess probability mass that triggers resampling
$D_{\mathrm{Branch}}(q,p\mid \boldsymbol{X}_{1:\gamma-1})
=\sum_{\tilde{x}_{\gamma}} q(\boldsymbol{X}_{1:\gamma-1}\tilde{x}_{\gamma}) (1-h_\gamma)$. Then we have}
        &= D_\text{Branch}(q, p|\boldsymbol{X}_{1:\gamma-1})\cdot 1 \cdot \frac{p(\boldsymbol{X}_{1:\gamma}) - q(\boldsymbol{X}_{1:\gamma})}{D_\text{Branch}(p, q|\boldsymbol{X}_{1:\gamma-1})} +  \\
        &\sum_{\tilde{x}_{\gamma-1}}D_\text{Branch}(q, p|\boldsymbol{X}_{1:\gamma-2}\tilde{x}_{\gamma-1})
  (1-\frac{D_\text{Branch}(p, q|\boldsymbol{X}_{1:\gamma-2}\tilde{x}_{\gamma-1})}{D_\text{Branch}(q, p|\boldsymbol{X}_{1:\gamma-2}\tilde{x}_{\gamma-1})})
  P_\text{res}(x_{\gamma-1})P_\text{res}(x_{\gamma})
  % \frac{\text{max}\{p(x^\gamma) - q(x^\gamma), 0\}}{D_\text{Branch}(p, q|X_{1:\gamma-1})}
  \\
  \intertext{From \Cref{def:asym} and \Cref{thm:hierarchy}, we know that $\sum_{\tilde{x}_{\gamma-1}} D_\text{Branch}(q, p|\boldsymbol{X}_{1:\gamma-2}\tilde{x}_{\gamma-1}) - D_\text{Branch}(p, q|\boldsymbol{X}_{1:\gamma-2}\tilde{x}_{\gamma-1}) = D_\text{Branch}(q, p|\boldsymbol{X}_{1:\gamma-2})$. Then we have}
&= \frac{D_\text{Branch}(q, p|\boldsymbol{X}_{1:\gamma-1})}{D_\text{Branch}(p, q|\boldsymbol{X}_{1:\gamma-1})} \cdot (p(\boldsymbol{X}_{1:\gamma}) - q(\boldsymbol{X}_{1:\gamma})) +  \\
& D_\text{Branch}(q, p|\boldsymbol{X}_{1:\gamma-2}) 
   \cdot \frac{p(\boldsymbol{X}_{1:\gamma-1}) - q(\boldsymbol{X}_{1:\gamma-1})}{D_\text{Branch}(p, q|\boldsymbol{X}_{1:\gamma-2})} 
   \cdot \frac{p(\boldsymbol{X}_{1:\gamma}) - q(\boldsymbol{X}_{1:\gamma})}{D_\text{Branch}(p, q|\boldsymbol{X}_{1:\gamma-1})} \\
     \intertext{We know from \Cref{def:asym} and \Cref{theorem:branch} that $p(\boldsymbol{X}_{1:\gamma}) - q(\boldsymbol{X}_{1:\gamma}=D_\text{Branch}(p, q|\boldsymbol{X}_{1:\gamma-1}) - D_\text{Branch}\bigl(q, p|\boldsymbol{X}_{1:\gamma-1}\bigr)$. Then we have}
     &= \frac{D_\text{Branch}(q, p|\boldsymbol{X}_{1:\gamma-1})}{D_\text{Branch}(p, q|\boldsymbol{X}_{1:\gamma-1})} \cdot (p(\boldsymbol{X}_{1:\gamma}) - q(\boldsymbol{X}_{1:\gamma})) +  \\ &\frac{\bigl(D_\text{Branch}(p,q \mid \boldsymbol{X}_{1:\gamma-1}) - D_\text{Branch}(q,p \mid \boldsymbol{X}_{1:\gamma-1})\bigr)}{D_\text{Branch}(p,q \mid \boldsymbol{X}_{1:\gamma-1})} \cdot (p(\boldsymbol{X}_{1:\gamma}) - q(\boldsymbol{X}_{1:\gamma}))\\
  &=p(\boldsymbol{X}_{1:\gamma}) - q(\boldsymbol{X}_{1:\gamma}) 
\end{align}

\subsubsection{General Proof}
\label{proof:naive-proof}

\begin{lemma}[Rejection-Resampling Sum Reduction (Tokenwise)]
\label{lemma:resample_sum_reduction}
Let \( 0 < m < \gamma \) be such that the acceptance ratios satisfy:
\begin{equation}
r(x_{\gamma}) > 1,\; r(x_{\gamma-1}) > 1,\; \dots,\; r(x_{\gamma-m+1}) > 1,\quad r(x_{\gamma-m}) \leq 1.
\end{equation}
Then, the total probability of obtaining the output via resampling over the last \( m \) positions is:
\begin{equation}
\sum_{i=0}^{m-1}P(x_{\gamma-i} \text{ is rejected}) \prod_{j=0}^{i}P(\boldsymbol{X}_{1:\gamma-j} \text{ is resampled}) = p(\boldsymbol{X}_{1:\gamma}) - q(\boldsymbol{X}_{1:\gamma}).
\end{equation}
\end{lemma}

\begin{proof}
We begin by defining auxiliary quantities to simplify the notation. For \( i = 0, 1, \dots, m \), let
\begin{equation}
\begin{aligned}
\Delta^+_i &:= D_{\mathrm{Branch}}(q, p \mid \boldsymbol{X}_{1:\gamma-i}), \\
\Delta^-_i &:= D_{\mathrm{Branch}}(p, q \mid \boldsymbol{X}_{1:\gamma-i}),
\end{aligned}
\end{equation}
where \( \Delta^-_i \) quantifies the probability mass to be corrected due to overestimation by \( q \), and \( \Delta^+_i \) represents the mass available to be allocated from alternate paths.

Define also the recursive product term:
\begin{equation}
P_i := \prod_{j=0}^{i} \frac{\Delta^+_j - \Delta^-_j}{\Delta^+_{j+1}}, \qquad \text{for } 0 \le i \le m-1.
\end{equation}

Using these, the rejection-resample contribution becomes:
\begin{equation}
\begin{aligned}
&\sum_{i=1}^{m-1}P(x_{\gamma-i} \text{ is rejected}) \prod_{j=0}^{i}P(\boldsymbol{X}_{1:\gamma-j} \text{ is resampled})\\
& = \sum_{i=1}^{m-1} \Delta^-_i P_i + (\Delta^+_{m-1} - \Delta^-_{m-1}) P_{m-1}.
\end{aligned}
\end{equation}

Now observe the recurrence:
\begin{equation}
\Delta^+_{k+1} P_{k+1} = (\Delta^+_k - \Delta^-_k) P_k,
\end{equation}
which implies:
\begin{equation}
(\Delta^+_k - \Delta^-_k) P_k = \Delta^+_{k+1} P_{k+1}.
\end{equation}

We apply this recurrence in reverse to simplify equation (1) by telescoping the sum:
\begin{equation}
\begin{aligned}
\sum_{i=1}^{m-1} \Delta^-_i P_i + (\Delta^+_{m-1} - \Delta^-_{m-1}) P_{m-1}
&= \sum_{i=1}^{m-2} \Delta^-_i P_i + \Delta^+_{m-1} P_{m-1} \\
&= \sum_{i=1}^{m-3} \Delta^-_i P_i + \Delta^+_{m-2} P_{m-2} \\
&\,\;\vdots \\
&= \Delta^+_1 P_1 \\
&= \Delta^+_0 - \Delta^-_0 \\
&= p(\boldsymbol{X}_{1:\gamma}) - q(\boldsymbol{X}_{1:\gamma}),
\end{aligned}
\end{equation}
where the final equality follows from the definition:

\begin{equation}
\Delta^+_0 - \Delta^-_0 = D_{\mathrm{Branch}}(q, p \mid \boldsymbol{X}_{1:\gamma}) - D_{\mathrm{Branch}}(p, q \mid \boldsymbol{X}_{1:\gamma}) = p(\boldsymbol{X}_{1:\gamma}) - q(\boldsymbol{X}_{1:\gamma}).
\end{equation}

This completes the proof.
\end{proof}

\begin{lemma}[No Resampling of Earlier Prefixes (Tokenwise)]
\label{lemma:no_resample_prefix_Tokenwise}
Let \( \boldsymbol{X}_{1:\gamma} = [x_1, x_2, \dots, x_\gamma] \) be a token block, and suppose that for some index \( m \), the acceptance ratios satisfy:

\begin{equation}
r(x_{\gamma}) > 1,\; r(x_{\gamma-1}) > 1,\; \dots,\; r(x_{\gamma-m+1}) > 1,\quad r(x_{\gamma-m}) \leq 1.
\end{equation}
Then for all \( t \leq \gamma - m \), the resampling probability satisfies:
\begin{equation}
P(\boldsymbol{X}_{1:t} \text{ is resampled}) = 0.
\end{equation}
\end{lemma}

\begin{proof}
We use the resampling probability formula:
\begin{equation}
P_{\text{res}}(\boldsymbol{X}_{1:t})
= \frac{\max\left\{p(\boldsymbol{X}_{1:t}) - q(\boldsymbol{X}_{1:t}),\, 0\right\}}{\max\left\{D_{\mathrm{Branch}}(p,q \mid \boldsymbol{X}_{1:t}),\; D_{\mathrm{Branch}}(q,p \mid \boldsymbol{X}_{1:t})\right\}}.
\end{equation}

At position \( t = \gamma - m \), we are given that the acceptance probability

\begin{equation}
r(x_{\gamma - m}) = \min\left\{1, \frac{p(\boldsymbol{X}_{1:\gamma - m})}{q(\boldsymbol{X}_{1:\gamma - m})} \right\} \leq 1,
\end{equation}
implying \( p(\boldsymbol{X}_{1:\gamma - m}) < q(\boldsymbol{X}_{1:\gamma - m}) \). Therefore,
\begin{equation}
p(\boldsymbol{X}_{1:\gamma - m}) - q(\boldsymbol{X}_{1:\gamma - m}) \leq 0,
\end{equation}
and hence:

\begin{equation}
P_{\text{res}}(\boldsymbol{X}_{1:\gamma - m}) = 0.
\end{equation}

This completes the proof.
\end{proof}

\begin{theorem}[Lossless]
\begin{equation}
P(\texttt{yield } \boldsymbol{X}_{1:\gamma}) = p(\boldsymbol{X}_{1:\gamma}).
\end{equation}
\end{theorem}

% \begin{proof}
% \begin{align*}
% P(\texttt{yield } \boldsymbol{X}_{1:\gamma})
% &= P(\boldsymbol{X}_{1:\gamma} \text{ is accepted}) + \sum_{i=1}^m P(x_{\gamma-i} \text{ is rejected}) \prod_{j=0}^i P(\boldsymbol{X}_{1:\gamma-j} \text{ is resampled}) \\
% &= q(\boldsymbol{X}_{1:\gamma}) + \bigl(p(\boldsymbol{X}_{1:\gamma}) - q(\boldsymbol{X}_{1:\gamma})\bigr) + 0 \\
% &= p(\boldsymbol{X}_{1:\gamma}).
% \end{align*}

\begin{proof}
The total probability is the sum of the acceptance and resampling paths. We analyze two cases based on the relative probabilities.

\textbf{Case 1: $p(\boldsymbol{X}_{1:\gamma}) < q(\boldsymbol{X}_{1:\gamma})$}
In this case, the acceptance probability for the draft is $\frac{p(\boldsymbol{X}_{1:\gamma})}{q(\boldsymbol{X}_{1:\gamma})}$. The probability of generating $\boldsymbol{X}_{1:\gamma}$ via resampling is 0, as there is no probability deficit to recover.
\begin{equation}
\begin{aligned}
P(\texttt{yield } \boldsymbol{X}_{1:\gamma}) &= P(\boldsymbol{X}_{1:\gamma} \text{ is accepted}) + P(\boldsymbol{X}_{1:\gamma} \text{ is resampled}) \\
&= q(\boldsymbol{X}_{1:\gamma}) \cdot \frac{p(\boldsymbol{X}_{1:\gamma})}{q(\boldsymbol{X}_{1:\gamma})} + 0 \\
&= p(\boldsymbol{X}_{1:\gamma}).
\end{aligned}
\end{equation}

\textbf{Case 2: $p(\boldsymbol{X}_{1:\gamma}) \ge q(\boldsymbol{X}_{1:\gamma})$}
Here, the acceptance probability for the draft is $1$. The resampling path must compensate for the probability deficit. Per  \cref{lemma:resample_sum_reduction} and \cref{lemma:no_resample_prefix_Tokenwise}, the total probability of all relevant resampling paths is exactly $p(\boldsymbol{X}_{1:\gamma}) - q(\boldsymbol{X}_{1:\gamma})$.
\begin{equation}
\begin{aligned}
P(\texttt{yield } \boldsymbol{X}_{1:\gamma}) &= P(\boldsymbol{X}_{1:\gamma} \text{ is accepted}) + P(\boldsymbol{X}_{1:\gamma} \text{ is resampled}) \\
&= q(\boldsymbol{X}_{1:\gamma}) \cdot 1 + \bigl(p(\boldsymbol{X}_{1:\gamma}) - q(\boldsymbol{X}_{1:\gamma})\bigr) \\
&= p(\boldsymbol{X}_{1:\gamma}).
\end{aligned}
\end{equation}

These two cases cover all probability events. In both cases, the total probability correctly recovers $p(\boldsymbol{X}_{1:\gamma})$, proving the method is lossless.

\end{proof}

% \tianyu{
% \begin{align*}
% P(\boldsymbol{X}_{1:\gamma}\text{ is yielded})
% &= P(\boldsymbol{X}_{1:\gamma} \text{is accepted}) + \sum_{i=1}^m P(x_{\gamma-i} \text{is rejected}) \prod_{j=0}^i P(\boldsymbol{X}_{1:\gamma-j} \text{is resampled}) \\
% &= q(\boldsymbol{X}_{1:\gamma}) + \bigl(p(\boldsymbol{X}_{1:\gamma}) - q(\boldsymbol{X}_{1:\gamma})\bigr) + 0 \\
% &= p(\boldsymbol{X}_{1:\gamma}).
% \end{align*}
% }
% \end{proof}

\subsection{Lossless of Hierarchical Speculative Decoding}
\label{app: lossless}
% \hengl{Still need the proof from old algorithm}
% \hengl{proof is need to align with new notation, add details and fix missing scenarios.}

\subsubsection{Illustrative Example}
\label{app: example_cap}

Let \(p(\cdot)\) be the target and \(q(\cdot)\) the draft. For a prefix \(\boldsymbol{X}_{1:t}\),
\[
r(\boldsymbol{X}_{1:t})\ :=\ \frac{p(\boldsymbol{X}_{1:t})}{q(\boldsymbol{X}_{1:t})},\qquad
r(\boldsymbol{X}_{a+1:b}\mid \boldsymbol{X}_{1:a})\ :=\ \frac{p(\boldsymbol{X}_{a+1:b}\mid \boldsymbol{X}_{1:a})}{q(\boldsymbol{X}_{a+1:b}\mid \boldsymbol{X}_{1:a})},
\]
so \(r(\boldsymbol{X}_{1:b})=r(\boldsymbol{X}_{1:a})\,r(\boldsymbol{X}_{a+1:b}\mid \boldsymbol{X}_{1:a})\). Let \(m\) be the last (largest) index \(<\gamma\) at which the running maximum of \(r(\boldsymbol{X}_{1:t})\) is attained and exceeds \(1\); let \(n<m\) be the previous such index (two-peak case).

As \cref{def:max_ratios}, define the capped ratio at the end of the draft as
\[
r^{\!*}(\boldsymbol{X}_{1:\gamma})\ :=\ \min\{r(\boldsymbol{X}_{1:m}),1\}\,r(\boldsymbol{X}_{m+1:\gamma}\mid \boldsymbol{X}_{1:m})\ =\ r(\boldsymbol{X}_{m+1:\gamma}\mid \boldsymbol{X}_{1:m})\ \le 1,
\]
and the \emph{accept} term
\[
A_\gamma\ :=\ q(\boldsymbol{X}_{1:\gamma})\,r^{\!*}(\boldsymbol{X}_{1:\gamma}).
\]
We will also use three \emph{resample} contributions: \(T_\gamma\) (at level \(\gamma\)), \(T_m\) (at level \(m\)), and \(T_n\) (at level \(n\)).

\textbf{two-peak example: \(n<m<\gamma\)}
From~\cref{def:max_ratios}, we have \(r(\boldsymbol{X}_{1:n})>1\), then \(r(\boldsymbol{X}_{1:m})>r(\boldsymbol{X}_{1:n})\), and no larger value occurs in \((m,\gamma)\).
This forces \(r(\boldsymbol{X}_{n+1:m}\mid \boldsymbol{X}_{1:n})>1\); otherwise \(m\) could not be a new maximum.

\textbf{Step 1: accept + top-level resample}
Since \(r^{\!*}(\boldsymbol{X}_{1:\gamma})=r(\boldsymbol{X}_{m+1:\gamma}\mid \boldsymbol{X}_{1:m})\le 1\),
\[
A_\gamma\ =\ q(\boldsymbol{X}_{1:\gamma})\,r(\boldsymbol{X}_{m+1:\gamma}\mid \boldsymbol{X}_{1:m})
\ =\ q(\boldsymbol{X}_{1:m})\,p(\boldsymbol{X}_{m+1:\gamma}\mid \boldsymbol{X}_{1:m}),
\qquad
T_\gamma\ =\ 0,
\]
so
\[
H_1\ :=\ A_\gamma+T_\gamma\ =\ q(\boldsymbol{X}_{1:m})\,p(\boldsymbol{X}_{m+1:\gamma}\mid \boldsymbol{X}_{1:m}).
\]
\emph{Intuition.} The suffix \(\boldsymbol{X}_{m+1:\gamma}\) is now under \(p\); the prefix \(\boldsymbol{X}_{1:m}\) is still under \(q\).

\textbf{Step 2: add the \(m\)-term}
Let \(R_{n\to m}:=r(\boldsymbol{X}_{n+1:m}\mid \boldsymbol{X}_{1:n})>1\). The resample at level \(m\) contributes
\[
T_m\ :=\ q(\boldsymbol{X}_{1:m})\,(R_{n\to m}-1)\,p(\boldsymbol{X}_{m+1:\gamma}\mid \boldsymbol{X}_{1:m}),
\]
hence
\[
H_2\ :=\ H_1+T_m
\ =\ R_{n\to m}\,q(\boldsymbol{X}_{1:m})\,p(\boldsymbol{X}_{m+1:\gamma}\mid \boldsymbol{X}_{1:m})
\ =\ q(\boldsymbol{X}_{1:n})\,p(\boldsymbol{X}_{n+1:\gamma}\mid \boldsymbol{X}_{1:n}).
\]
\emph{Intuition.} The block \(\boldsymbol{X}_{n+1:m}\) is converted to \(p\); only \(\boldsymbol{X}_{1:n}\) remains under \(q\).

\textbf{Step 3: add the \(n\)-term}
If \(r(\boldsymbol{X}_{1:n})>1\),
\[
T_n\ :=\ q(\boldsymbol{X}_{1:n})\,(r(\boldsymbol{X}_{1:n})-1)\,p(\boldsymbol{X}_{n+1:\gamma}\mid \boldsymbol{X}_{1:n}),
\qquad
H_3\ :=\ H_2+T_n\ =\ p(\boldsymbol{X}_{1:\gamma}).
\]
If instead \(r(\boldsymbol{X}_{1:n})\le 1\), then \(T_n=0\) and \(H_2=p(\boldsymbol{X}_{1:\gamma})\) already.

\paragraph{Intuition.}
Each nonzero term “tops up” the exact deficit of \(q\) on its block until the whole path is under \(p\). Thus
\[
\boxed{\,A_\gamma + T_\gamma + T_m + T_n\ =\ p(\boldsymbol{X}_{1:\gamma})\,}
\]
in this two-peak case, exhibiting the (lossless) invariance of the total probability under the HSD accept--resample rule.

\subsubsection{General Proof}
\label{app: lossless proof}
% This is a patch to introduce capping method instead of just naive one. Here are some example for the naive method. 

% This is a patch to introduce capping method instead of just naive one. With Proofs.

% This is a patch to introduce capping method instead of just naive one. With Proofs.

\begin{definition}[Sequence of Unique Capping Indices]
\label{def:unique_capping_indices}
For a given maximum sequence length $\gamma$, the sequence of maximum prefix ratio indices $(m(1), m(2), \ldots, m(\gamma))$ is generated according to Definition~\ref{def:max_ratios}. Let $\mathcal{U}$ be the set of unique values in the sequence of capping indices:
\begin{equation}
\mathcal{U} = \{ m(t) \mid 1 < t \leq \gamma \}
\end{equation}
The \textbf{Sequence of Unique Capping Indices}, denoted by $M^*$, is the ordered sequence of the elements in $\mathcal{U}$:
\begin{equation}
M^* = (m_1^*, \ldots, m_L^*)
\end{equation}
where $m_1^* < \ldots < m_L^*$ and $L$ is the total number of unique capping points.
\end{definition}

With these definitions, we can now establish the key properties of the prefix-capped joint ratio:

\begin{lemma}[Property of $r^*(\boldsymbol{X}_{1:i})$ between neighboring unique capping indices]
\label{lemma:r_star_property}
Let \(m^*_l\) and \(m^*_{l+1}\) be two consecutive unique capping indices, and suppose
\begin{equation}
m^*_l < i < m^*_{l+1}.
\end{equation}
For every such \(i\), we have $r^*(\boldsymbol{X}_{1:i}) \leq 1$.
\end{lemma}

\begin{lemma}[Property of $r^*(\boldsymbol{X}_{1:m^*_l})$ at unique capping indices]
\label{lemma:r_star_capping_property}
Let \(m^*_{l-1}\) and \(m^*_{l}\) be two consecutive unique capping indices, we have 
$$r^*(\boldsymbol{X}_{1:m^*_{l}}) = r(\boldsymbol{X}_{m^*_{l-1}+1:m^*_{l}}) > 1$$.
\end{lemma}

We now define the acceptance and resampling probability masses:

\begin{definition}[Accepted Probability Mass]
\label{def:accepted_mass}
The probability mass for accepting the full sequence \( \boldsymbol{X}_{1:\gamma} \) is:
\begin{equation}
P(\boldsymbol{X}_{1:\gamma} \, \text{is accepted}) = \min(1, r^*(\boldsymbol{X}_{1:\gamma}))\, q(\boldsymbol{X}_{1:\gamma}),
\end{equation}
\end{definition}

\begin{definition}[Resampling Probability Mass]
\label{def:resampling_mass}
Let \( \boldsymbol{X}_{1:\gamma} \) be a full sequence of length \( \gamma \), and let $M^* = (m_1^*, m_2^*, \dots, m_L^*)$ be its Sequence of Unique Capping Indices. The total probability mass under the draft \(q\) and target \(p\) of generating this sequence can be decomposed as:

\textbf{Total Generation Probability}
\begin{equation}
\begin{aligned}
P\bigl(\boldsymbol{X}_{1:\gamma} \text{ is generated}\bigr) &= P\bigl(\boldsymbol{X}_{1:\gamma}\text{ is accepted}\bigr) + P\bigl(\boldsymbol{X}_{1:\gamma}\text{ is resampled}\bigr) \\
&= \min\bigl(1,\,r^*(\boldsymbol{X}_{1:\gamma})\bigr)\;q(\boldsymbol{X}_{1:\gamma}) \\
&\quad + \sum_{l=1}^{L} \max\bigl(0,\,r(\boldsymbol{X}_{m^*_{l-1}+1:m^*_{l}})-1\bigr)\,q(\boldsymbol{X}_{1:m^*_{l}})\,p(\boldsymbol{X}_{m^*_{l}+1:\gamma}\mid \boldsymbol{X}_{1:m^*_{l}}) \\
&\quad + \max\bigl(0,\,r^*(\boldsymbol{X}_{1:\gamma})-1\bigr)\,q(\boldsymbol{X}_{1:\gamma-1})\,p(x_\gamma\mid \boldsymbol{X}_{1:\gamma-1})
\end{aligned}
\end{equation}
\end{definition}

We now establish the key lemma that characterizes the resampling probability mass:

\begin{lemma}[Hierarchical Resampling Probability Mass]
\label{lemma:hierarchical_resampling}
The total generation probability can be decomposed into acceptance and resampling masses as stated in Definition~\ref{def:resampling_mass}. Only unique capping indices contribute to resampling mass, and the explicit form for the resampling mass at each unique capping index is:
\begin{equation}
\begin{aligned}
    &P\bigl(\boldsymbol{X}_{1:m^*_l}\text{ is resampled}\bigr)\, p\bigl(\boldsymbol{X}_{m^*_l+1:\gamma} \mid \boldsymbol{X}_{1:m^*_l}\bigr) \\
&= \max\bigl(0,\,r(\boldsymbol{X}_{m^*_{l-1}+1:m^*_{l}})-1\bigr)\,q(\boldsymbol{X}_{1:m^*_{l}})\,p(\boldsymbol{X}_{m^*_{l}+1:\gamma}\mid \boldsymbol{X}_{1:m^*_{l}})
\end{aligned}
\end{equation}
\end{lemma}

To prove the lossless property, we introduce the segmented probability function:

\begin{definition}[Segmented Probability Function]
\label{def:segmented_probability}
For each $l \in \{1, \dots, L\}$, we define the segmented probability function $F_l$ as:
\begin{equation}
\begin{aligned}
F_l &= 
q\bigl(\boldsymbol{X}_{1:m^{*}_{l}}\bigr)\;\;p\bigl(\boldsymbol{X}_{m^{*}_{l}+1:\gamma}\mid \boldsymbol{X}_{1:m^{*}_{l}}\bigr) \\
&= \Bigl[\prod_{i=1}^{m^{*}_{l}} q(x_i \mid \boldsymbol{X}_{1:i-1})\Bigr]
   \Bigl[\prod_{i=m^{*}_{l}+1}^{\gamma} p(x_i \mid \boldsymbol{X}_{1:i-1})\Bigr],
\end{aligned}
\end{equation}

This function represents a hybrid probability measure that uses the draft distribution $q$ up to position $m^{*}_{l}$ and the target distribution $p$ for the remaining positions, where $\boldsymbol{X}_{1:0}$ is equal to the prefix.

\end{definition}

We establish the telescoping property of resampling mass:

\begin{lemma}[Telescoping of Resampling Mass]
\label{lemma:telescoping}
For each $l \in \{1, \dots, L\}$, the mass of the resampling at the unique capping index $m^*_l$ can be expressed as:
\begin{equation}
P\bigl(\boldsymbol{X}_{1:m^*_l}\text{ is resampled}\bigr) 
= F_{l-1} - F_{l}\,.
\end{equation}
\end{lemma}

\begin{proof}
we need to show that the resampling mass at the unique capping index $m(l)$ equals $F_{l-1} - F_l$.

\subparagraph{1. Express $F_{l-1}$ in terms of $F_l$.}
We have 
$$P\bigl(\boldsymbol{X}_{1:m^*_l}\text{ is resampled}\bigr)  = \bigl(\,r(\boldsymbol{X}_{m^*_{l-1}+1:m^*_{l}})-1\bigr)\,
q(\boldsymbol{X}_{1:m^*_{l}})\,
p(\boldsymbol{X}_{m^*_{l}+1:\gamma}\mid \boldsymbol{X}_{1:m^*_{l}})$$

First note
\[
q\bigl(\boldsymbol{X}_{1:m^{*}_{l+1}}\bigr)
= q\bigl(\boldsymbol{X}_{1:m^{*}_{l}}\bigr)\;
  q\bigl(\boldsymbol{X}_{m^{*}_l+1:m^{*}_{l+1}} \mid \boldsymbol{X}_{1:m^{*}_{l}}\bigr),
\]
and
\[
p\bigl(\boldsymbol{X}_{m^{*}_l+1:m^{*}_{l+1}} \mid \boldsymbol{X}_{1:m^{*}_{l}}\bigr)
= r\bigl(\boldsymbol{X}_{m^{*}_l+1:m^{*}_{l+1}}\bigr)\;
  q\bigl(\boldsymbol{X}_{m^{*}_l+1:m^{*}_{l+1}} \mid \boldsymbol{X}_{1:m^{*}_{l}}\bigr).
\]

Hence
\begin{align*}
F_{l-1}
&= q\bigl(\boldsymbol{X}_{1:m^{*}_{l}}\bigr)\;
   p\bigl(\boldsymbol{X}_{m^{*}_l+1:\gamma} \mid \boldsymbol{X}_{1:m^{*}_{l}}\bigr) \\[6pt]
&= q\bigl(\boldsymbol{X}_{1:m^{*}_{l}}\bigr)\;
   p\bigl(\boldsymbol{X}_{m^{*}_l+1:m^{*}_{l+1}} \mid \boldsymbol{X}_{1:m^{*}_{l}}\bigr)\;
   p\bigl(\boldsymbol{X}_{m^{*}_{l+1}+1:\gamma} \mid \boldsymbol{X}_{1:m^{*}_{l+1}}\bigr) \\[6pt]
&= q\bigl(\boldsymbol{X}_{1:m^{*}_{l}}\bigr)\;
   \Bigl[r(\boldsymbol{X}_{m^{*}_l+1:m^{*}_{l+1}})\,q(\boldsymbol{X}_{m^{*}_l+1:m^{*}_{l+1}} \mid \boldsymbol{X}_{1:m^{*}_{l}})\Bigr]\;
   p\bigl(\boldsymbol{X}_{m^{*}_{l+1}+1:\gamma} \mid \boldsymbol{X}_{1:m^{*}_{l+1}}\bigr) \\[6pt]
&= r\bigl(\boldsymbol{X}_{m^{*}_l+1:m^{*}_{l+1}}\bigr)\;
   \Bigl[q(\boldsymbol{X}_{1:m^{*}_{l}})\,q(\boldsymbol{X}_{m^{*}_l+1:m^{*}_{l+1}} \mid \boldsymbol{X}_{1:m^{*}_{l}})\Bigr]\;
   p\bigl(\boldsymbol{X}_{m^{*}_{l+1}+1:\gamma} \mid \boldsymbol{X}_{1:m^{*}_{l+1}}\bigr) \\[6pt]
&= r\bigl(\boldsymbol{X}_{m^{*}_l+1:m^{*}_{l+1}}\bigr)\;
   q\bigl(\boldsymbol{X}_{1:m^{*}_{l+1}}\bigr)\;
   p\bigl(\boldsymbol{X}_{m^{*}_{l+1}+1:\gamma} \mid \boldsymbol{X}_{1:m^{*}_{l+1}}\bigr) \\[3pt]
&= r\bigl(\boldsymbol{X}_{m^{*}_l+1:m^{*}_{l+1}}\bigr)\;F_{l}.
\end{align*}

\subparagraph{2. Compute the difference $F_{l-1} - F_l$.}
\begin{align*}
F_{l-1} - F_l
&= \Bigl[r(\boldsymbol{X}_{m^{*}_l+1:m^{*}_{l+1}})\;F_l\Bigr] - F_l \\[3pt]
&= \bigl(r(\boldsymbol{X}_{m^{*}_l+1:m^{*}_{l+1}}) - 1\bigr)\;F_l \\[3pt]
&= \bigl(r(\boldsymbol{X}_{m^{*}_l+1:m^{*}_{l+1}}) - 1\bigr)\;
   q\bigl(\boldsymbol{X}_{1:m^{*}_{l+1}}\bigr)\;
   p\bigl(\boldsymbol{X}_{m^{*}_{l+1}+1:\gamma} \mid \boldsymbol{X}_{1:m^{*}_{l+1}}\bigr) \\[3pt]
&= \bigl(r(\boldsymbol{X}_{m^*_{l-1}+1:m^*_{l}})-1\bigr)\,
   q(\boldsymbol{X}_{1:m^*_{l}})\,
   p(\boldsymbol{X}_{m^*_{l}+1:\gamma}\mid \boldsymbol{X}_{1:m^*_{l}}).
\end{align*}

This completes the proof that the resampling mass at segment $l$ equals $F_{l-1} - F_l$.
\end{proof}

\begin{theorem}[Lossless Recovery]
\label{thm:lossless_recovery}
Under the prefix-adaptive speculative decoding scheme, the total probability of generating any sequence $\boldsymbol{X}_{1:\gamma}$ equals the target distribution probability:
\begin{equation}
P\bigl(\boldsymbol{X}_{1:\gamma} \text{ is generated}\bigr) = p(\boldsymbol{X}_{1:\gamma}).
\end{equation}
\end{theorem}

\begin{proof}
From Lemma~\ref{lemma:hierarchical_resampling}, we have the total generation probability decomposition:
\begin{equation}
\begin{aligned}
P\bigl(\boldsymbol{X}_{1:\gamma} \text{ is generated}\bigr) &= P\bigl(\boldsymbol{X}_{1:\gamma}\text{ is accepted}\bigr) + P\bigl(\boldsymbol{X}_{1:\gamma}\text{ is resampled}\bigr) + P\bigl(x_\gamma\text{ is resampled}\bigr) \\
&= \min\bigl(1,\,r^*(\boldsymbol{X}_{1:\gamma})\bigr)\;q(\boldsymbol{X}_{1:\gamma}) \\
&\quad + \sum_{l=1}^{L} \max\bigl(0,\,r(\boldsymbol{X}_{m^*_{l-1}+1:m^*_{l}})-1\bigr)\,q(\boldsymbol{X}_{1:m^*_{l}})\,p(\boldsymbol{X}_{m^*_{l}+1:\gamma}\mid \boldsymbol{X}_{1:m^*_{l}}) \\
&\quad + \max\bigl(0,\,r^*(\boldsymbol{X}_{1:\gamma})-1\bigr)\,q(\boldsymbol{X}_{1:\gamma-1})\,p(x_\gamma\mid \boldsymbol{X}_{1:\gamma-1})
\end{aligned}
\end{equation}

From Lemma~\ref{lemma:telescoping}, we know that for each $l \in \{1, \dots, L\}$:
\begin{equation}
F_{l-1} - F_l = \bigl(r(\boldsymbol{X}_{m^*_{l-1}+1:m^*_{l}})-1\bigr)\,
q(\boldsymbol{X}_{1:m^*_{l}})\,
p(\boldsymbol{X}_{m^*_{l}+1:\gamma}\mid \boldsymbol{X}_{1:m^*_{l}})
\end{equation}
Therefore, we can rewrite the generation probability as:
\begin{equation}
\begin{aligned}
P\bigl(\boldsymbol{X}_{1:\gamma} \text{ is generated}\bigr) &= \min\bigl(1,\,r^*(\boldsymbol{X}_{1:\gamma})\bigr)\;q(\boldsymbol{X}_{1:\gamma}) \\
&\quad + \sum_{l=1}^L (F_{l-1} - F_{l}) \\
&\quad + \max\bigl(0,\,r^*(\boldsymbol{X}_{1:\gamma})-1\bigr)\,q(\boldsymbol{X}_{1:\gamma-1})\,p(x_\gamma\mid \boldsymbol{X}_{1:\gamma-1})
\end{aligned}
\end{equation}

Since $r^*(\boldsymbol{X}_{1:\gamma}) = \min\{r(\boldsymbol{X}_{1:m^*_L}), 1\}r(\boldsymbol{X}_{m^*_L+1:\gamma})$ and $r(\boldsymbol{X}_{1:m^*_L}) > 1$, we have $r^*(\boldsymbol{X}_{1:\gamma}) = r(\boldsymbol{X}_{m^*_L+1:\gamma})$.

Case 1: If $r(\boldsymbol{X}_{m^*_L+1:\gamma}) \leq 1$, then:
\begin{equation}
\begin{aligned}
&\min\bigl(1,\,r^*(\boldsymbol{X}_{1:\gamma})\bigr)\;q(\boldsymbol{X}_{1:\gamma}) + \max\bigl(0,\,r^*(\boldsymbol{X}_{1:\gamma})-1\bigr)\,q(\boldsymbol{X}_{1:\gamma-1})\,p(x_\gamma\mid \boldsymbol{X}_{1:\gamma-1}) \\
&= r(\boldsymbol{X}_{m^*_L+1:\gamma})\;q(\boldsymbol{X}_{1:\gamma}) + 0 \\
&= r(\boldsymbol{X}_{m^*_L+1:\gamma})\;q(\boldsymbol{X}_{1:\gamma}) \\
&= q(\boldsymbol{X}_{1:m^*_L})\;p(\boldsymbol{X}_{m^*_L+1:\gamma} \mid \boldsymbol{X}_{1:m^*_L}) \\
&= F_L
\end{aligned}
\end{equation}

Case 2: If $r(\boldsymbol{X}_{m^*_L+1:\gamma}) > 1$, then there would be another unique capping index beyond $m^*_L$, contradicting the definition of $m^*_L$ as the last unique capping index. Therefore, we must have $r(\boldsymbol{X}_{m^*_L+1:\gamma}) \leq 1$, and thus:
\begin{equation}
\min\bigl(1,\,r^*(\boldsymbol{X}_{1:\gamma})\bigr)\;q(\boldsymbol{X}_{1:\gamma}) + \max\bigl(0,\,r^*(\boldsymbol{X}_{1:\gamma})-1\bigr)\,q(\boldsymbol{X}_{1:\gamma-1})\,p(x_\gamma\mid \boldsymbol{X}_{1:\gamma-1}) = F_L
\end{equation}

Therefore, we have:
\begin{equation}
\begin{aligned}
P\bigl(\boldsymbol{X}_{1:\gamma} \text{ is generated}\bigr) &= F_L + \sum_{l=1}^L (F_{l-1} - F_{l}) \\
&= F_L + (F_0 - F_1) + (F_1 - F_2) + \cdots + (F_{L-1} - F_L) \\
&= F_L + F_0 - F_L \\
&= F_0
\end{aligned}
\end{equation}

Now we evaluate $F_0$. From Definition~\ref{def:segmented_probability}, we have:
\begin{equation}
F_0 = q(\boldsymbol{X}_{1:m^*_0})\, p(\boldsymbol{X}_{m^*_0+1:\gamma} \mid \boldsymbol{X}_{1:m^*_0})
\end{equation}

By our convention, $m^*_0 = 0$, so:
\begin{equation}
F_0 = q(\boldsymbol{X}_{1:0})\, p(\boldsymbol{X}_{1:\gamma} \mid \boldsymbol{X}_{1:0}) = 1 \cdot p(\boldsymbol{X}_{1:\gamma}) = p(\boldsymbol{X}_{1:\gamma})
\end{equation}

Therefore:
\begin{equation}
\boxed{P\bigl(\boldsymbol{X}_{1:\gamma} \text{ is generated}\bigr) = p(\boldsymbol{X}_{1:\gamma})}
\end{equation}

This completes the proof of lossless recovery.
\end{proof}

\subsubsection{A Extended Explaination of Capped Ratio}
\label{app:capped_lossless_proof}
Let \( r(x_1), r(x_2\mid x_{1}), \dots, r(x_t\mid \boldsymbol{X}_{1:t-1}) \in \mathbb{R}_{> 0} \) be a sequence of ratios. 

Define the cumulative product up to index \( t \) as:
\begin{equation}
r(\boldsymbol{X}_{1:t}) = \prod_{i=1}^{t} r(x_i\mid\boldsymbol{X}_{1:i-1}),
\end{equation}
where $\boldsymbol{X}_{1:0}$ is equal to the prefix.

Let \( j^* \) be the last index (up to \( k \)) such that:
\begin{equation}
j^* = \max \left\{ j \leq k \,\middle|\; r(x_j\mid\boldsymbol{X}_{1:j-1}) > 1 \text{ and } \prod_{i=1}^{j} r(x_i\mid\boldsymbol{X}_{1:i-1}) > 1 \right\}
\end{equation}

Then the capped cumulative product \( \tilde{R}_k \) is given by:
\begin{equation}
r*(\boldsymbol{X}_{1:t}) = 
\left( \prod_{i=1}^{j^*} r(x_i\mid\boldsymbol{X}_{1:i-1}) \right) \cdot \left( \prod_{i=j^*+1}^{k} r(x_i\mid\boldsymbol{X}_{1:i-1}) \right)
\end{equation}

This ensures that the cumulative product is capped at the last index \( j^* \) such that the individual ratio \( r(x_{j^*\mid\boldsymbol{X}_{1:j^*-1}}) > 1 \) and the cumulative product up to that point also exceeds 1.

When $\gamma$ is 3, lets show simplest example to show the recovery of target probability.

\begin{equation}
\begin{aligned}
P\left( \boldsymbol{X} _{1: 3} \text { is accepted }\right) 
& =q(\boldsymbol{X}_{1:3})
\end{aligned}
\end{equation}

\begin{equation}
\begin{aligned}
&P\left( \boldsymbol{X} _{1: 3} \text { is resampled}\right)  =\sum_{i=0}^{\gamma = 3} P\left(x_\gamma, x_{\gamma-1}, \ldots, x_{\gamma-i} \text { are resampled } \mid \boldsymbol{X} _{\gamma-i+1}\right) \\
& = D_{\operatorname{Branch}}^*\left(q, p \mid \boldsymbol{X} _{1: 3}\right) \cdot \frac{\max((r(x_3) - 1)q(\boldsymbol{X}_{1:3}), 0)}{D_{\operatorname{Branch}}^*\left(q, p \mid \boldsymbol{X} _{1: 3}\right)} \\
&+D_{\operatorname{Branch}}^*\left(q, p \mid \boldsymbol{X} _{1: 2}\right) \cdot \frac{\max((r(x_2) - 1)q(\boldsymbol{X}_{1:2}), 0)}{D_{\operatorname{Branch}}^*\left(q, p \mid \boldsymbol{X} _{1: 2}\right)} \cdot p(x_3|\boldsymbol{X}_{1:2})  \\
&+D_{\operatorname{Branch}}^*\left(q, p \mid x_1\right) \cdot \frac{\max((r(x_1) - 1)q(x_{1}), 0)}{D_{\operatorname{Branch}}^*\left(q, p \mid x_1\right)} \cdot p(x_3|\boldsymbol{X}_{1:2})p(x_2|x_{1}) \\ 
\end{aligned}
\end{equation}

Let's take $\gamma=3$ as an example, only if $r(\boldsymbol{X}_{1:3})>1$,  the resampled portion of probability mass is needed. Suppose $r(\boldsymbol{X}_{1:2})>1$ with $r(x_1)>1$ and $r(x_2)<1$:

\begin{equation}
\begin{aligned}
&= p(x_3|\boldsymbol{X}_{1:2})p(x_2|{x_1})q(x_1) - q(\boldsymbol{X}_{1:3}) + 0 \\ &+p(x_1)p(x_2|x_{1})p(x_3|\boldsymbol{X}_{1:2}) - q(x_1)p(x_2|x_{1})p(x_3|\boldsymbol{X}_{1:2})
\\
&= p(\boldsymbol{X}_{1:3}) - q(\boldsymbol{X}_{1:3})\\
\end{aligned}
\end{equation}

\subsection{Expected Number of Accepted Tokens}
\label{sec:length}
We conduct efficiency analysis based on the expected acceptance length $\mathbb{E}[\tau]$.
For a given draft length $\gamma$, the expected number of accepted tokens for the tokenwise speculative decoding \cite{leviathan2023fast}, blockwise verification \cite{sun2024block}, and our HSD  are as follows:

\begin{lemma} \textbf{Expected Number of Accepted Tokens} (See~\Cref{app:expected-length} for proof.)
\begin{equation}
\small
\begin{aligned}
\setlength{\abovedisplayskip}{2pt}
\setlength{\belowdisplayskip}{2pt}
\!\!\!\mathbb{E}[\tau]_{\text{token}} = \sum_{i=1}^\gamma \prod_{k=1}^{i} h_k^{\text{token}},
\mathbb{E}[\tau]_{\text{block}} = \sum_{i=1}^\gamma \left[1 - \prod_{k=i}^\gamma \left(1 - h_k^{\text{block}}\right)\right],
\mathbb{E}[\tau]_{\text{branch}} = \sum_{i=1}^\gamma \left[1 - \prod_{k=i}^\gamma \left(1 - h_k\right)\right]
\end{aligned}
\end{equation}
% Here,$h_k^{\text{token}}$ means token wise verification \cite{leviathan2023fast}, $h_k^{\text{block}}$ denotes block verification \cite{sun2024block}, and $h_k$ is the acceptance probability for HSD.
\end{lemma}

We establish Theorem~\ref{thm:expected_length_comparison}, which guarantees that HSD is more efficient than other lossless methods:

\begin{theorem} 
\label{thm:expected_length_comparison}
    HSD and Blockwise Achieves Better Expected Number of Accepted Tokens
\begin{equation}
\mathbb{E}[\boldsymbol{\tau}]_{\text {branch }} \geq \mathbb{E}[\boldsymbol{\tau}]_{\text {block }} \geq \mathbb{E}[\boldsymbol{\tau}]_{\text {token }}\end{equation}
where equality holds in both inequalities if and only if $\gamma = 1$.
(See \Cref{app:compare} for proof.)
\end{theorem}
\begin{wrapfigure}{r}{0.3\textwidth}
\setlength{\abovecaptionskip}{0pt} 
\setlength{\belowcaptionskip}{0pt} 
% \vspace{-0.7cm}
\centering
\includegraphics[width=0.3\textwidth]{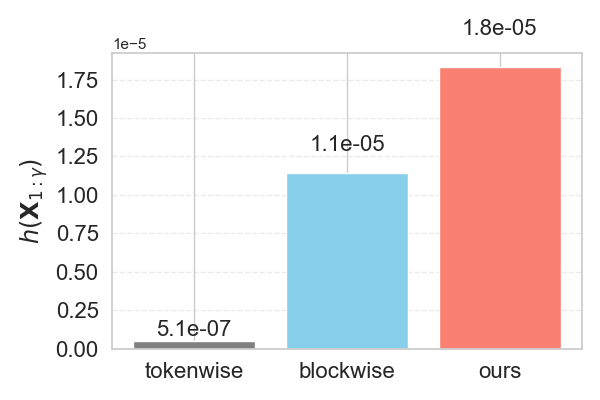}
\caption{\small{The average acceptance probability of the entire draft ($\tau=\gamma$) on GSM8K.}}
\label{fig:accept}
\vspace{-0.9cm}
\end{wrapfigure}
We reveal that limitations on acceptance probability in each method directly cause the gap from the ideal case w.r.t. expected accepted tokens. Let \( r(x_t) = \frac{p(x_t)}{q(x_t)} \). The acceptance probability of the entire draft  $h_{\gamma}$ is ideally \( \min\left\{ \prod_{t=1}^\gamma r(x_t), 1 \right\} \). In contrast, tokenwise acceptance is \( h_{\text{token}} = \prod_{t=1}^\gamma \min\{r(x_t), 1\} \), blockwise adopts \( h_{\text{block}} = \min\{ 1, r_\gamma, r_{\gamma-1}r_\gamma, \dots, r_1r_2\cdots r_\gamma \} \) (see \Cref{lem:suffix-minimum}), and HSD uses \( h_{\text{ours}} = \min\left\{ \min\left\{ \prod_{t=1}^{m(x_{\gamma})} r(x_t), 1 \right\}\prod_{t=m(x_{\gamma})+1}^{\gamma} r(x_t), 1 \right\}  \). See the average acceptance probability $h_\gamma$ on GSM8K in Fig.~\ref{fig:accept}.

Let \( \tau \in \{0, 1, \dots, \gamma\} \) denote the number of accepted tokens in a decoding attempt. Since \( \tau \) is a non-negative, integer-valued random variable, the tail-sum identity applies with lattice spacing \( a = 1 \).

\begin{lemma}[Tail Expectation]
Let $X$ be a non–negative random variable with values in \( \{na : n = 0, 1, 2, \dots\} \) for some \( a > 0 \). Then:
\begin{equation}
\mathbb{E}[X] = a \sum_{k=1}^{\infty} \Pr(X \ge k). \label{eq:tail-lattice}
\end{equation}
\end{lemma}

\begin{proof}
Start with the right-hand side:
\begin{equation}
\begin{aligned}
a \sum_{k=1}^{\infty} \Pr(X \ge k a)
&= a \sum_{k=1}^{\infty} \sum_{\ell \ge k} \Pr(X = \ell a) \\
&= a \sum_{\ell = 1}^{\infty} \Pr(X = \ell a) \sum_{k=1}^{\ell} 1 \\
&= \sum_{\ell = 1}^{\infty} \ell a \cdot \Pr(X = \ell a) = \mathbb{E}[X]. 
\end{aligned}
\end{equation}
\end{proof}

\subsubsection{Expected Token Length Derivation}
\label{app:expected-length}

\subsubsection*{Token Wise Speculative Decoding}

Referring to \emph{Block‑wise Verification}~\cite{sun2024block}, the authors prove that it achieves a longer expected token length than the token‑wise verification~\cite{leviathan2023fast} (see Appendix~B.2 in ~\cite{sun2024block}).

\subsubsection*{Hierarchical Speculative Decoding}

Let \( \eta_1, \dots, \eta_\gamma \sim \mathcal{U}(0, 1) \) be the random draws used in verification. The accepted length is defined as:
\begin{equation}
\tau := \max \left\{ i \leq \gamma \,:\, \eta_i \leq h_i \right\}, \label{eq:tau-def}
\end{equation}
where \( h_i \) is the acceptance probability at step \( i \). By the tail-sum identity:
\begin{equation}
\mathbb{E}[\tau] = \sum_{i=1}^{\gamma} \Pr(\tau \ge i). \label{eq:tail-sum}
\end{equation}

If we define the event \( S_i := \{ \eta_i \leq h_i \} \), and assume independence of the draws, then:
\begin{equation}
\Pr(\tau \ge i) = 1 - \prod_{k=i}^{\gamma} (1 - h_k). \label{eq:tail-prob}
\end{equation}

Substituting into Equation~\eqref{eq:tail-sum}, we obtain:
\begin{equation}
\mathbb{E}[\tau] = \sum_{i=1}^{\gamma} \left[1 - \prod_{k=i}^{\gamma} (1 - h_k)\right]. \label{eq:expected-length-tail}
\end{equation}

\subsubsection*{Blockwise Verification}

In Algorithm 2 (blockwise decoding), the decoding continues even if some \( \eta_i > h_i^{\text{block}} \); the resampling happens only at the end. Therefore, the token count \( \tau \) still satisfies the same form.

Let \( h_i^{\text{block}} \) be the acceptance probability at step \( i \) computed via blockwise rules, and define events:
\begin{equation}
S_i := \{\eta_i \le h_i^{\text{block}}\}, \quad \text{so } \Pr(\overline{S_i}) = 1 - h_i^{\text{block}}.
\end{equation}

We then have:
\begin{equation}
\Pr(\tau \ge i) = 1 - \prod_{k=i}^{\gamma} (1 - h_k^{\text{block}}),
\end{equation}
and hence the expected number of accepted tokens under blockwise decoding is:
\begin{equation}
\mathbb{E}[\tau]_{\text{block}} = \sum_{i=1}^{\gamma}
\left[ 1 - \prod_{k=i}^{\gamma} (1 - h_k^{\text{block}}) \right] \label{eq:expected-length-blockwise}
\end{equation}

\subsubsection{Token Length Comparison}
\label{app:compare}

We re-express the acceptance probability to compare token length between block-wise speculative decoding and our method (~\Cref{def:acceptance_probability}). This yields a more precise comparison via the directional divergence expressions~\Cref{eq:d_star} and 
 ~\Cref{eq:side_by_side_align}.

\paragraph{Capped Branch Divergence Difference}
The difference of capped branch divergence is calculated as:
\begin{align}
    &D^*_{\text{Branch}}\left(p, q \mid \boldsymbol{X}_{1: t}\right) 
    - D^*_{\text{Branch}}\left(q, p \mid \boldsymbol{X}_{1: t}\right) \notag\\
    &= \sum_{x_{t+1}} \left( r^*(\boldsymbol{X}_{1:t+1}) - 1 \right) 
    q(\boldsymbol{X}_{1:t}) \notag\\
    &= \sum_{x_{t+1}}\bigl( \min\{r(\boldsymbol{X}_{0:m(t+1)}),1\}r(\boldsymbol{X}_{m(t+1)+1:t+1})-1\bigr)q(\boldsymbol{X}_{0:m(t+1)})q(\boldsymbol{X}_{m(t+1)+1:t+1})  \notag \\
    &= \sum_{x_{t+1}}\bigl(r(\boldsymbol{X}_{m(\boldsymbol{X}_{1:t+1})+1:t+1})-1\bigr)q(\boldsymbol{X}_{1:t+1}) 
    \label{eq:d_difference}
\end{align}

\paragraph{Branch Acceptance Probability}
Combine equations (\ref{eq:d_difference}), the acceptance ratio of hierarchical speculative decoding is:
\begin{equation}
\begin{aligned}
   h_t^{\mathrm{branch}}
   &=\frac{D^*_{\text{Branch}}\left(p, q \mid \boldsymbol{X}_{1: t}\right) }{D^*_{\text{Branch}}\left(q, p \mid \boldsymbol{X}_{1: t}\right) }\\&=\frac{D^*_{\text{Branch}}\left(p, q \mid \boldsymbol{X}_{1: t}\right) }{D^*_{\text{Branch}}\left(p, q \mid \boldsymbol{X}_{1: t}\right) +\sum (1-r(\boldsymbol{X}_{m(\boldsymbol{X}_{1:t+1})+1:t+1}))q(\boldsymbol{X}_{1:t+1})}\\
&=\frac{\sum[r(\boldsymbol{X}_{m(\boldsymbol{X}_{1:t+1})+1:t+1})-1]_+}{\sum[r(\boldsymbol{X}_{m(\boldsymbol{X}_{1:t+1})+1:t+1})-1]_++\sum(1-r(\boldsymbol{X}_{m(\boldsymbol{X}_{1:t+1})+1:t+1}))} \label{eq: h_branch}
\end{aligned}
\end{equation}

where $[a]_+$ is equal to $\max\{a,0\}$

\paragraph{Blockwise Acceptance Ratio}
Algorithm 2 (blockwise decoding), blockwise keeps an internal clamp
\(
p_{t} = \min\{p_{t-1}\,r(x_{t}|\boldsymbol{X}_{1:t-1}),1\}
\), which could be simplified based on Suffix–minimum characterization of \(p_t\)
%-------------------------------------------------
% Lemma: suffix–minimum characterization of p_t
%-------------------------------------------------
\begin{lemma}[Suffix–minimum characterization of \(p_t\)]
\label{lem:suffix-minimum}
Let \(\{r_i\}_{i=1}^{\infty}\subseteq[0,\infty)\) and define the sequence \(\{p_t\}_{t\ge 0}\) recursively by
\begin{equation}
   p_0 \;=\; 1,
   \qquad
   p_t \;=\; \min\!\bigl\{\,p_{t-1}\,r_t,\;1\bigr\},
   \quad t\ge 1.
\end{equation}
Then for every \(t\ge 0\)
\begin{equation}\label{eq:pt_closed}
   p_t \;=\; \min_{0\le s\le t}\;
             \prod_{i=s+1}^{t} r_i,
   \qquad
   (\text{with the empty product for }s=t\text{ equal to }1).
\end{equation}
Equivalently,
\begin{equation}
   p_t \;=\; \min\!\bigl\{\,1,\;r_t,\;r_{t-1}r_t,\,\dots,\;r_1r_2\cdots r_t\bigr\}.
\end{equation}
\end{lemma}

\begin{proof}
We prove \eqref{eq:pt_closed} by induction on \(t\).

\medskip
\noindent\textbf{Base case (\(t=0\)).}
For \(t=0\) the right–hand side becomes
\begin{equation}
   \min_{0\le s\le 0}(\text{empty product}) = 1 = p_0,
\end{equation}
so the claim holds.

\medskip
\noindent\textbf{Inductive step.}
Assume \eqref{eq:pt_closed} holds for some \(t-1\ge 0\).  Using the
recurrence,
\begin{equation}
   p_t \;=\; \min\!\bigl\{\,1,\;p_{t-1}\,r_t\bigr\}.
\end{equation}
By the induction hypothesis,
\begin{equation}
   p_{t-1} = \displaystyle\min_{0\le s\le t-1}\prod_{i=s+1}^{t-1} r_i.
\end{equation}
Substituting,
\begin{equation}
   p_t
   = \min\!\Bigl\{\,1,\;
                   \bigl[\min_{0\le s\le t-1}
                         \prod_{i=s+1}^{\,t-1} r_i\bigr] r_t
             \Bigr\}.
\end{equation}
Multiplying every candidate product in the inner minimum by \(r_t\)
and then taking the outer minimum yields exactly all suffix products
\begin{equation}
   \prod_{i=s+1}^{t} r_i
\end{equation}
for \(s = 0,\dots ,t-1\), together with the empty product \(1\) for
\(s=t\).  Hence \eqref{eq:pt_closed} holds for \(t\), completing
the induction. And obviously, $p_t<r(\boldsymbol{X}_{start:t})$, where $start\in(1,t-1)$ 
\end{proof}

\begin{equation}
\begin{aligned}
&\!\!\!\!\!h^{\text{block}}_t\!\!=\!\!\frac{\displaystyle\sum_{x_{t+1}} (p_{t}r(x_{t+1}|\boldsymbol{X}_{1:t})-1)_{+}\;q(x_{t+1}\mid \boldsymbol{X}_{1:t})}{\displaystyle\sum_{x_{t+1}} (p_{t}r(x_{t+1}|\boldsymbol{X}_{1:t})-1)_{+}\;q(x_{t+1}\mid \boldsymbol{X}_{1:t})\;+\;1-p_{t}} \\
&\!\!\!\!\!\!\!=\!\!\frac{\displaystyle\sum_{x_{t+1}}(\min\!\bigl\{r(x_{t+1}),\!r(\boldsymbol{X}_{t:t+1}),\!r(\boldsymbol{X}_{t-1:t+1}),\dots,\!r(\boldsymbol{X}_{1:t+1})\bigr\}\!-\!1)_{\!+\!\!}\;q(x_{t+1}\!\mid\! \boldsymbol{X}_{1:t})}{\displaystyle\sum_{x_{t+1}}(\min\!\bigl\{\,r(x_{t+1})\!,\!\;r(\boldsymbol{X}_{t:t+1})\!,\!\;r(\boldsymbol{X}_{t-1:t+1})\!,\!\dots\!,\!r(\boldsymbol{X}_{1:t+1})\bigr\}\!-\!1)_{\!+\!}\!q(x_{t+1}\!\mid\! \boldsymbol{X}_{1:t})\!+\!1\!-\!p_{t}}
\end{aligned}
\end{equation}

Since $\min\!\bigl\{\,r(x_{t+1}),\;r(\boldsymbol{X}_{t:t+1}),\;r(\boldsymbol{X}_{t-1:t+1}),\,\dots,\;r(\boldsymbol{X}_{1:t+1})\bigr\}\leq r(\boldsymbol{X}_{m(\boldsymbol{X}_{1:t+1})+1:t+1})$,
\begin{equation}
\begin{aligned}
    h^{\text{block}}_t&
    \leq\frac{\displaystyle
        \sum_{x_{t+1}} (r(\boldsymbol{X}_{m(\boldsymbol{X}_{1:t+1})+1:t+1})-1)_{+}\;
                   q(x_{t+1}\mid \boldsymbol{X}_{1:t})}
       {\displaystyle
        \sum_{x_{t+1}} (r(\boldsymbol{X}_{m(\boldsymbol{X}_{1:t+1})+1:t+1})-1)_{+}\;
                   q(x_{t+1}\mid \boldsymbol{X}_{1:t})
        \;+\;1-p_{t}}\\
&\leq\frac{\displaystyle
        \sum_{x_{t+1}} (r(\boldsymbol{X}_{m(\boldsymbol{X}_{1:t+1})+1:t+1})-1)_{+}}
       {\displaystyle
        \sum_{x_{t+1}} (r(\boldsymbol{X}_{m(\boldsymbol{X}_{1:t+1})+1:t+1})-1)_{+}
        \;+\;1-p_{t}}
\end{aligned}
\end{equation}

From equation \ref{eq: h_branch}:
\begin{equation}
\begin{aligned}
    \sum_{x_{t+1}}(1-r(\boldsymbol{X}_{m(\boldsymbol{X}_{1:t+1})+1:t+1}))&=q(\boldsymbol{X}_{m(t+1)+1:t})-p(\boldsymbol{X}_{m(t+1)+1:t})\\&=(1-r(\boldsymbol{X}_{m(t+1)+1:t}))q(\boldsymbol{X}_{m(t+1)+1:t})\\&\leq (1-p_t)q(\boldsymbol{X}_{m(t+1)+1:t}) \\&\leq (1-p_t)
\end{aligned}
\end{equation}
Since
\begin{equation}
\begin{aligned}
   h_t^{\mathrm{branch}}
   &=\frac{\sum_{x_{t+1}}[r(\boldsymbol{X}_{m(\boldsymbol{X}_{1:t+1})+1:t+1})-1]_+}{\sum_{x_{t+1}}[r(\boldsymbol{X}_{m(\boldsymbol{X}_{1:t+1})+1:t+1})-1]_++\sum_{x_{t+1}}(1-r(\boldsymbol{X}_{m(\boldsymbol{X}_{1:t+1})+1:t+1}))} \\
   &\geq \frac{\sum_{x_{t+1}}[r(\boldsymbol{X}_{m(\boldsymbol{X}_{1:t+1})+1:t+1})-1]_+}{\sum_{x_{t+1}}[r(\boldsymbol{X}_{m(\boldsymbol{X}_{1:t+1})+1:t+1})-1]_++1-p_t}\\
   &\geq h_t^{\mathrm{block}}
\end{aligned}
\end{equation}

\newpage

\subsection{Extended Experiments}
\label{ap:experiment_extenstion}

\noindent\textbf{Result Robustness} To prove the robustness of our experiments and guarantee fair comparison, we conduct additional experiments with different methods as shown in Table \ref{tab:std}. We observe that our method demonstrates stable performance and exceeds both tokenwise and blockwise methods on average.

\begin{table}[t]
    \centering
        \caption{Comparison of different algorithm performance on GSM8K with Qwen-2.5. We list the average and standard deviation across 5 runs with different seeds.}
    \begin{tabular}{cccc}
\toprule
    \textbf{Method}&\textbf{Tokenwise}&\textbf{Blockwise}&\textbf{Ours} \\
   \midrule
         Block Efficiency& 6.40$\pm$0.10& 6.51$\pm$0.09&6.64$\pm$0.04 \\         Decoding Speed&31.52$\pm$0.06 &31.70$\pm$0.05 &32.61$\pm$0.02\\
\bottomrule
    \end{tabular}

    \label{tab:std}
\end{table}

\noindent\textbf{Verification of Task Performance} 
We compare our method with the token-wise approach on GSM8K. As shown in Table \ref{tab:acc}, our method achieves equivalent (or better) accuracy among different model sizes, demonstrating the preserved distributional fidelity.
% \noindent\textbf{Performance Discussion on Eagle-3.} We also present recent work Eagle-3 performance on GSM8K. Unlike the losslessness of our methods, according to its official implementation, Eagle-3 employs lossy verification (\url{https://github.com/SafeAILab/EAGLE/blob/main/eagle/model/utils.py#L398}, where they directly set the probability of target model q as 1.0.). This means that if models are not fine-tuned on specific datasets, their performance will drop accordingly. In Table \ref{tab:eagle}, we observe a significant performance decline when evaluating Eagle3 on GSM8K, with accuracy dropping from 74.81\% (baseline) to 69.47\% (Eagle-3), representing a 5.34 percentage reduction.

\begin{table}[!t]
\centering
\caption{Comparison of task performance across model sizes and methods.}
\begin{tabular}{lcccc}
\toprule
\textbf{Metric} & \textbf{Method} & \textbf{72B} & \textbf{32B} & \textbf{14B} \\
\midrule
\multirow{2}{*}{GSM8K (Accuracy)} 
  & Tokenwise & 0.8213   & 0.8213   & 0.8327 \\
  & HSD       & 0.8517   & 0.8479   & 0.8327 \\
\bottomrule
\end{tabular}
\label{tab:acc}
\end{table}

% \noindent\textbf{Experiments on Llama}
% We expanded our evaluation beyond Qwen2.5 and conducted additional experiments using Llama-3.1-70B-Instruct and Llama-3.1-8B-Instruct pair (non-quantized version, with model weights distributed on 8 H20 GPUs). Evaluations were conducted on four diverse tasks—GSM8K, CNN/DailyMail, HumanEval, and FLORES—covering mathematical reasoning, summarization, code generation, and multilingual translation. HSD consistently improves both block efficiency and end-to-end decoding speed over tokenwise and blockwise verification across all tasks, demonstrating that our method generalizes across architectures and domains. Full results are listed in Table \ref{tab:llama-performance}.

% \noindent\textbf{Eagle Integration Experiments}
% To further demonstrate compatibility with advanced decoding systems, we integrated HSD into the state-of-the-art EAGLE-3-LLaMa3.1-Instruct-8B (draft length $\gamma=7$ by default) by replacing its tokenwise verifier in Table \ref{tab:Eagle_integration}. These results show that replacing the baseline verifier with HSD yields consistent improvements in both decoding speed and block efficiency. 

\noindent\textbf{Capped Prefix Ratio Ablation Study}
We conducted ablations on the role of the capped prefix ratio in Algorithm 2 (HSD), which is essential to preserve distributional fidelity, as shown in Table \ref{tab:Capping_ablation}. Removing capping (i.e., directly using the uncapped ratio to compute divergences) yields a slight increase in efficiency, but at the cost of varying degrees of performance degradation (which may or may not be obvious depending on the task).

% \noindent\textbf{Multi-draft experiments with Llama-3 model pair.}
% We provide an expanded set of multi-draft evaluations comparing HSD Multi-draft to Tokenwise Multi-draft, using a different model family from the Qwen experiments in Table \ref{tab:multi-draft-integration}. Specifically, we evaluate the Llama-3.1-70B-Instruct / Llama-3.1-8B-Instruct pair (non-quantized), with model weights distributed across 8 H20 GPUs. 

% \begin{table}[!t]
% \centering
% \caption{Performance comparison with Llama-3 model pair.}
% \begin{tabular}{lcc|cc|cc|cc}
% \toprule
% \multirow{2}{*}{Component} & \multicolumn{2}{c|}{GSM8K} & \multicolumn{2}{c|}{CNNdaily} & \multicolumn{2}{c|}{HumanEval} & \multicolumn{2}{c}{Flores} \\
%  & BE & DS & BE & DS & BE & DS & BE & DS \\
% \midrule
% Tokenwise & 6.83 & 8.41 & 5.13 & 6.89 & 8.01 & 13.44 & 4.92 & 7.81 \\
% Blockwise & 7.32 & 8.87 & 5.49 & 7.26 & 8.33 & 13.56 & 5.05 & 8.04 \\
% HSD       & 7.43 & 9.18 & 5.57 & 7.44 & 8.40 & 13.94 & 5.22 & 8.39 \\
% \bottomrule
% \end{tabular}
% \label{tab:llama-performance}
% \end{table}

% \begin{table}[!t]
% \centering
% \caption{Eagle-3 integration.}
% \begin{tabular}{lccc}
% \toprule
% \textbf{Dataset} & \textbf{Method} & \textbf{Speed (tokens/s)} & \textbf{BE} \\
% \midrule
% \multirow{2}{*}{GSM8K} 
%   & Baseline & 71.59   & 3.40  \\
%   & +HSD       & 80.49   & 3.55   \\
% \bottomrule
% \end{tabular}
% \label{tab:Eagle_integration}
% \end{table}

\begin{table}[!t]
\centering
\caption{Ablation on capping mechanism.}
\begin{tabular}{lcccc}
\toprule
\textbf{Dataset} & \textbf{Method} & \textbf{ACC} & \textbf{BE} &\textbf{DS} \\
\midrule
  GSM8K & HSD  & 84.40±1.75\%   & 6.76±0.05 & 33.63±0.53\\
  HumanEval & HSD       & 80.61±0.69\%   & 5.60±0.06 & 29.15±0.43 \\
    GSM8K & HSD + Capping  & 84.96±0.93\%   & 	6.63±0.06 & 32.73±0.55\\
  HumanEval & HSD  + Capping      & 82.47±1.15\%  & 5.45±0.08& 27.54±0.48  \\
\bottomrule
\end{tabular}
\label{tab:Capping_ablation}
\end{table}

% \begin{table}[!t]
% \centering
% \caption{Multi-draft integration using Llama-3.1-70B-Instruct / Llama-3.1-8B-Instruct pair.}
% \begin{tabular}{lcccc}
% \toprule
% \textbf{Dataset} & \textbf{BE} & \textbf{DS} \\
% \midrule
%   Tokenwise & 8.72 &10.21 \\
%   HSD & 9.00    & 11.02  \\
% \bottomrule
% \end{tabular}
% \label{tab:multi-draft-integration}
% \end{table}
% \clearpage

\subsection{Python Implementation}
\label{ap:python}
We provide the Python implementation of our Hierarchical Speculative Decoding (HSD) algorithm in \Cref{ls:hsd}, which builds upon the token-wise speculative decoding approach from Hugging Face \cite{wolf-etal-2020-transformers} Transformers v4.46.3, shown in \Cref{ls:sd} for comparison. Following Hugging Face, our implementation eliminates the use of an explicit for-loop by leveraging an equivalent masking mechanism: we perform parallel sampling across all positions to determine whether to accept or reject subsequences of varying lengths, and then select the longest accepted prefix as the final output.

\begin{listing}[h]%
\caption{Tokenwise Speculative Decoding (SD) {\tt SD.py}}%
\label{ls:sd}%
\begin{lstlisting}[language=Python]
import torch

def SD(candidate_input_ids, candidate_logits, new_logits):
    """
    Args:
        candidate_input_ids (Tensor): Token IDs from the draft model. Shape: [batch_size, seq_len]
        candidate_logits (Tensor): Logits from the draft model. Shape: [batch_size, seq_len, vocab_size]
        new_logits (Tensor): Logits from the target model. Shape: [batch_size, seq_len, vocab_size]
    Returns:
        n_matches (int): Number of accepted tokens from the draft model.
        valid_tokens (Tensor): Accepted token prefix with one new token sampled. Shape: [batch_size, n_matches+1]
    """

    # Convert logits to probabilities
    q = candidate_logits.softmax(dim=-1)
    p = new_logits.softmax(dim=-1)

    candidate_length = candidate_logits.shape[1]
    new_candidate_input_ids = candidate_input_ids[:, -candidate_length:]

    # Extract token-wise probabilities for the candidate tokens
    q_i = q[:, torch.arange(candidate_length), new_candidate_input_ids].squeeze(1)
    p_i = p[:, torch.arange(candidate_length), new_candidate_input_ids].squeeze(1)

    probability_ratio = p_i / q_i
    is_accepted = torch.rand_like(probability_ratio) <= probability_ratio

    # assuming batch size = 1
    n_matches = ((~is_accepted).cumsum(dim=-1) < 1).sum()  # this is `n` in algorithm 1

    # Next token selection: if there is a rejection, adopt the resampling distribution.
    if n_matches < candidate_length:
        p_n_plus_1 = p[:, n_matches, :]
        q_n_plus_1 = q[:, n_matches, :]
        p_prime = torch.clamp((p_n_plus_1 - q_n_plus_1), min=0)
        p_prime.div_(p_prime.sum())
    else:
        p_prime = p[:, n_matches, :]

    # Ensure we don't generate beyond max_len or an EOS token.
    if is_done_candidate[0] and n_matches == candidate_length:

        # Output length is assumed to be `n_matches + 1`. Since we won't generate another token with the target model
        # due to acceptance on EOS we fix `n_matches`
        n_matches -= 1
        valid_tokens = candidate_input_ids[:, -candidate_length:]

    else:
        # Next token selection: if there is a rejection, adjust the distribution from the main model before sampling.
        # The selected tokens include the matches (if any) plus the next sampled tokens
        if n_matches > 0:
            if n_matches < candidate_length:
                valid_tokens = candidate_input_ids[:, -candidate_length:n_matches - candidate_length]
                if not stop(valid_tokens, scores=None):
                    t = torch.multinomial(p_prime, num_samples=1)
                    valid_tokens = torch.cat(
                        (valid_tokens, t), dim=-1)
                else:
                    n_matches = n_matches-1
            else:
                valid_tokens = candidate_input_ids[:, -candidate_length:]
                if not stop(valid_tokens, scores=None):
                    t = torch.multinomial(p_prime, num_samples=1)
                    valid_tokens = torch.cat(
                    (valid_tokens, t), dim=-1)
                else:
                    n_matches = n_matches -1
        else:
            t = torch.multinomial(p_prime, num_samples=1)
            valid_tokens = t

    return valid_tokens, n_matches

    


\end{lstlisting}
\end{listing}

\begin{listing}[t]%
\caption{Hierarchical Speculative Decoding (HSD) {\tt HSD.py}}%
\label{ls:hsd}%
\begin{lstlisting}[language=Python]
import torch

def HSD(candidate_input_ids, candidate_logits, new_logits):
    """
    Args:
        candidate_input_ids (Tensor): Token IDs from the draft model. Shape: [batch_size, seq_len]
        candidate_logits (Tensor): Logits from the draft model. Shape: [batch_size, seq_len, vocab_size]
        new_logits (Tensor): Logits from the target model. Shape: [batch_size, seq_len, vocab_size]
    Returns:
        n_matches (int): Number of accepted tokens from the draft model.
        valid_tokens (Tensor): Accepted token prefix with one new token sampled. Shape: [batch_size, n_matches+1]
    """
    
    # Convert logits to probabilities
    q = candidate_logits.softmax(dim=-1)
    p = new_logits.softmax(dim=-1)
    candidate_length = candidate_logits.shape[1]
    new_candidate_input_ids = candidate_input_ids[:, -candidate_length:]
    
    # Extract token-wise probabilities for the candidate tokens
    q_i = q[:, torch.arange(candidate_length), new_candidate_input_ids].squeeze(1)
    p_i = p[:, torch.arange(candidate_length), new_candidate_input_ids].squeeze(1)
    
    # Compute cumulative joint probabilities for draft and target model
    q_prev = torch.roll(q_i, shifts=1, dims=1)
    q_prev[:, 0] = 1.0
    q_cumprod = torch.exp(torch.log(q_prev).cumsum(dim=1)).unsqueeze(-1)
    q_next = q_cumprod * q[:, :candidate_length]
    p_prev = torch.roll(p_i, shifts=1, dims=1)
    p_prev[:, 0] = 1.0
    p_cumprod = torch.exp(torch.log(p_prev).cumsum(dim=1)).unsqueeze(-1)
    
    # Constrain p_cumprod with q_cumprod for computing the capped resampling distribution
    ratio = p_cumprod / q_cumprod
    previous_max = 1
    new_p_previous = torch.ones_like(p_cumprod).to(p_cumprod.device)
    for k in range(candidate_length):
        if ratio[:, k] > previous_max:
            previous_max = ratio[:, k]
        new_p_previous[:, k] = p_cumprod[:, k] / previous_max
    p_next = new_p_previous * p[:, :candidate_length]
    
    # Construct resampling distribution p'
    diffs = p_next - q_next
    p_plus = torch.clamp(diffs, min=0.0)
    p_minus = torch.clamp(-diffs, min=0.0)
    p_primes = p_plus / torch.maximum(p_plus.sum(dim=-1, keepdim=True), p_minus.sum(dim=-1, keepdim=True))
    
    # Step-back probability: reject prefix with 1 - mass of p'
    step_back_probs = 1 - p_primes.sum(dim=-1)
    step_back = torch.rand_like(step_back_probs) < step_back_probs
    
    # Find first position to stop (from the end)
   if step_back.all():
        stop_positions = 0
    else:
        stop_positions = candidate_length - n_matches - 1 - torch.flip(~step_back, [-1]).max(-1, keepdim=True)[1]
    
    # Mask to decide which tokens are accepted
    select = torch.zeros_like(step_back).to(step_back.device)

    # apply cumprod on the ratio instead of the raw probabilities to avoid underflow
    probability_ratio = (p_i / q_i).cumprod(1).unsqueeze(-1)
    is_accepted = torch.rand_like(probability_ratio) <= probability_ratio
    
    # only decide to accept or not at the last position based on the joint probability ratio
    # assign 0 to all positions when the full draft is rejected, otherwise assign 1 to the rest of the positions
    select[torch.arange(p_primes.shape[0]), stop_positions] = ~is_accepted[:, -1:]
    is_accepted = 1 - torch.cumsum(select, dim=-1)

    #### assume batch_size=1 for the current implementation
    n_matches = is_accepted.sum().item()
\end{lstlisting}
\end{listing}
\clearpage
\begin{listing}[h]
\ContinuedFloat
\caption{Hierarchical Speculative Decoding {\tt HSD.py} (cont.)}
\begin{lstlisting}
    if is_done_candidate[:] and n_matches == candidate_length:
        # Output length is assumed to be `n_matches + 1`. Since we won't generate another token with the target model
        # due to acceptance on EOS we fix `n_matches`
        n_matches -= 1
        # valid_tokens = new_candidate_input_ids[:, : n_matches + 1]
        valid_tokens = candidate_input_ids[:, -candidate_length:]

    else:
        # Next token selection: if there is a rejection, adjust the distribution from the main model before sampling.
        gamma = candidate_length
        p_n_plus_1 = p[:, candidate_length, :]
        if n_matches < gamma:
            p_prime = p_primes[:, n_matches]
            p_prime = p_prime/p_prime.sum(-1, keepdim=True)
        else:
            p_prime = p_n_plus_1

        # The selected tokens include the matches (if any) plus the next sampled tokens
        # because if n_matches=0, we add one resampled token for sure, if n_matches=10, we add one more for sure
        # as well, because the previous if checked not stop and n_matches-candidate_length will be 0 causing problem
        if n_matches > 0 and n_matches<candidate_length:
            valid_tokens = candidate_input_ids[:, -candidate_length:n_matches-candidate_length]
            if not stop(candidate_input_ids[:, :n_matches-candidate_length], scores=None):
                t = torch.multinomial(p_prime, num_samples=1)
                valid_tokens = torch.cat(
                    (valid_tokens, t), dim=-1)
            else:
                n_matches = n_matches-1
        else:
            t = torch.multinomial(p_prime, num_samples=1)
            if n_matches==0:
                valid_tokens = t
            else:
                valid_tokens = candidate_input_ids[:, -candidate_length:]
                valid_tokens = torch.cat(
                    (valid_tokens, t), dim=-1)

    return valid_tokens, n_matches
\end{lstlisting}
\end{listing}

\clearpage

\subsection{Integration with Recursive Reject Sampling in the Multi-Draft Setup}
We demonstrate in \Cref{alg:backward_rrs} that our HSD algorithm is compatible with existing lossless multi-draft verification methods, exemplified by Recursive Reject Sampling (RRS) with replacement~\cite{yang2024multi}. Notably, independently sampled parallel draft sequences do not guarantee the existence of an additional draft sequence that shares the accepted subsequence as its prefix.

\begin{algorithm}[h]
\setstretch{1}
\caption{Hierarchical Speculative Sampling with Recursive Rejection Sampling}
\label{alg:backward_rrs}
\begin{algorithmic}[1]
\REQUIRE 
Draft tokens: $\boldsymbol{X}^k_{1:t}=\{x^k_1, ..., x^k_\gamma \}_{k=1}^K$;\\ 
Target probabilities for all draft tokens:  $\{p(\cdot),..., p(\cdot|\boldsymbol{X}^k_{1:\gamma})\}_{k=1}^K$;\\
Draft probabilities for all draft tokens:  $\{q(\cdot),..., q(\cdot|\boldsymbol{X}^k_{1:\gamma})\}_{k=1}^K$;\\

\STATE Initialize $\tau = 0$;
\STATE Initialize $\{x_i^1\}_1^\gamma$;

\FOR{$k$ \textbf{in} $1:K$}
    \IF{$\boldsymbol{X}_{1:\tau} = \boldsymbol{X}_{1:\tau}^k$}
        \FOR{$j$ \textbf{in} $\tau+1:\gamma$}
           \STATE \textbf{Set} $x_j=x^k_j$    \hfill {\color{blue}\textit{\#select draft $\boldsymbol{X}^k_{\tau+1:\gamma}$} for verification}
        \ENDFOR
        \STATE ~
        \FOR{$t$ \textbf{in} $\gamma:\tau+1$}
            \STATE Compute acceptance probability $h_t$ from \Cref{def:acceptance_probability} based on the corresponding probabilities for the draft tokens: $\{x_{\tau+1}, ..., x_\gamma\}$
            
            \STATE Sample $\eta_t \sim U(0,1)$
        
            \IF{$h_t \geq \eta_t$}
                 \STATE \textbf{Set} $\tau = t$
                \STATE \textbf{break}
            \ELSE
                 \STATE \textbf{Set} $\tau = t-1$
                \STATE \textbf{continue}
            \ENDIF
        \ENDFOR
    \ELSE
        \STATE \textbf{continue}  \hfill {\color{blue}\textit{\#skip  draft $\boldsymbol{X}^k_{1:\gamma}$} due to prefix mismatch}
    \ENDIF
    \STATE ~
    \IF{$\tau = \gamma$}
        \STATE Sample token from $p(\cdot|\boldsymbol{X}_{1:\gamma})$ \hfill {\color{blue}\textit{\#accept the entire selected draft and sample a bonus token}}
         \STATE \textbf{break}
    \ELSE
        \STATE 
           \textbf{Compute} $P^*_{\text{res}}(\cdot \mid 
           \boldsymbol{X}_{1:\tau})$;\\
             \STATE \textbf{Set} $p(\cdot|\boldsymbol{X}_{1:\tau})=P^*_{\text{res}}(\cdot \mid \boldsymbol{X}_{1:\tau})$;  \hfill {\color{blue}\textit{\#set  $P^*_{\text{res}}(\cdot \mid \boldsymbol{X}_{1:\tau})$} as new target distribution} \\
            \textbf{Set} $r(\cdot|\boldsymbol{X}_{1:\tau})=\frac{P^*_{\text{res}}(\cdot \mid \boldsymbol{X}_{1:\tau})}{q(\tilde{x}|\boldsymbol{X}_{1:\tau})}$  \hfill {\color{blue}\textit{\#set  $r(\cdot \mid \boldsymbol{X}_{1:\tau})$} as new probability ratio}
    \ENDIF
        \STATE ~
\ENDFOR 

Sample token from $P^*_{\text{res}}(\cdot \mid \boldsymbol{X}_{1:\tau})$
\ENSURE $[\boldsymbol{X}_{1:\tau}, \text{token}]$
\end{algorithmic}

\end{algorithm}

\clearpage

\subsection{Computation Efficiency}
\label{app:computation}

We begin by noting that the computational cost of the verification stage in HSD is \textit{effectively as efficient as that of tokenwise verification} in practice.

While there are minor differences in the computational cost of verification---whether using any of the three verification methods---these differences are \textbf{insignificant in practice} compared to the reduction in target model forward passes. Indeed, \textbf{block efficiency (or equivalently, the acceptance rate) remains the most meaningful metric for evaluating performance}. A detailed complexity analysis is provided in the revised version below:

Both HSD (Eqs.\ 17--20 in our paper) and blockwise verification (Eqs.\ 4--5 in \cite{sun2024block}) require summing over the vocabulary to compute the acceptance probability at each position. Therefore, HSD introduces \textbf{no theoretical overhead} compared to blockwise verification.

Moreover, our implementation is \textbf{more efficient} than that of blockwise verification (Appendix A in \cite{sun2024block} ). By leveraging an equivalent masking mechanism, we eliminate the for-loop and compute probabilities for all positions in parallel. This makes HSD nearly as efficient as tokenwise verification. While tokenwise verification only sums over the vocabulary at a rejected position, \textbf{our HSD computation is fully parallelized via tensor operations across both the vocabulary and the draft length $\gamma$, which is typically much smaller than the vocabulary size $\mathcal{V}$}.

Importantly, the computational cost of verification is negligible relative to the reduction in target model forward passes, which is the main bottleneck in verification. Below, we compare the verification cost with the forward-pass reduction for a batch size of~1.

Since the vocabulary size $\mathcal{V}$ is much larger than the draft length $\gamma$, the main cost of HSD arises from computing branch divergences, which requires only $4 \gamma \mathcal{V}$ FLOPs:

\begin{enumerate}
    \item $r^*(\mathbf{X}_{1:t}) - 1$ $\rightarrow$ $\gamma \mathcal{V}$ FLOPs  
    \item Selecting $(r^*(\mathbf{X}_{1:t}) - 1 > 0)$ via the $\max$ operator $\rightarrow$ $\gamma \mathcal{V}$ FLOPs  
    \item Multiplication by $q$ $\rightarrow$ $\gamma \mathcal{V}$ FLOPs  
    \item Summation over the vocabulary $\rightarrow$ $\gamma \mathcal{V}$ FLOPs  
\end{enumerate}

For Qwen-2.5 with $|\mathcal{V}| = 151{,}643$ and $\gamma = 10$, this amounts to approximately \textbf{5.8M FLOPs}.

In contrast, \textbf{the forward-pass FLOPs per new token} in large language models are orders of magnitude larger, even with KV-cache inference. Considering the main contributions:

\begin{itemize}
    \item $Q \cdot K$ dot products  
    \item Attention score $\times V$  
    \item All projection/MLP FLOPs  
    \item Ignoring softmax, LayerNorm, rotary embeddings, and bias  
\end{itemize}

The per-token FLOPs can be approximated as:
\[
\text{FLOPs per new token} \approx L \cdot \Big[4 d L_{\text{past}} + 4 d^2 + 2 d d_{\text{ff}} \Big],
\]
where $L$ is the number of transformer layers, $d$ is the hidden size, $d_{\text{ff}}$ is the MLP intermediate size, and $L_{\text{past}}$ is the number of cached tokens.

Using a context length $L_{\text{past}} = 1024$, the per-token FLOPs for Qwen2.5 models are:

\begin{itemize}
    \item Qwen2.5-0.5B: \textbf{0.374 GFLOPs}  
    \item Qwen2.5-72B: \textbf{62.915 GFLOPs}  
\end{itemize}

As shown in \Cref{tab:main_table} in our paper, all methods achieve block efficiency larger than 2. Consequently, the cost of HSD verification is negligible relative to the reduction in target model forward passes.

To directly quantify the overhead of the verification stage, we evaluate verification cost on 100 GSM8K problems using GPTQ-quantized 8-bit Qwen2.5-72B-Instruct and Qwen2.5-0.5B-Instruct as target and draft models on a single H200 GPU. As shown in \Cref{tab:runtime}, the verification stage consistently accounts for \textbf{less than 1\%} of the total decoding time. At the same time, the vast majority of runtime is spent in the draft and target forward passes. Here, the draft forward pass accounts for about 24\% of the runtime, and the target forward pass accounts for about 72\% in both blockwise and HSD. The verification stage of HSD is about 20\% faster than that of blockwise.

\begin{table}[h!]
\centering
\caption{Runtime breakdown of Blockwise and HSD.}
\label{tab:runtime}
\small
\begin{tabular}{lrrrr}
\toprule
\textbf{Component} & 
\textbf{Blockwise Mean (ms/token)} & 
\textbf{Blockwise \%} & 
\textbf{HSD Mean (ms/token)} & 
\textbf{HSD \%} \\
\midrule
\textbf{Total}            & \textbf{34.168} & \textbf{100.00\%} & \textbf{33.788} & \textbf{100.00\%} \\
Prefill          & 0.913 & 2.67\% & 0.915 & 2.71\% \\
Draft Forward    & 8.210 & 24.03\% & 7.865 & 23.28\% \\
Target Forward   & 24.690 & 72.26\% & 24.695 & 73.09\% \\
KV Cache Input   & 0.007 & 0.02\% & 0.005 & 0.01\% \\
KV Cache Output  & 0.070 & 0.20\% & 0.067 & 0.20\% \\
Logits Processing & 0.093 & 0.27\% & 0.103 & 0.31\% \\
Verification     & 0.160 & 0.47\% & 0.127 & 0.37\% \\
Other            & 0.025 & 0.08\% & 0.012 & 0.03\% \\
\bottomrule
\end{tabular}
\end{table}

\end{document}